\documentclass{article}

     \PassOptionsToPackage{numbers, compress}{natbib}
\usepackage[preprint]{neurips_2026}

\usepackage[utf8]{inputenc} % allow utf-8 input
\usepackage[T1]{fontenc}    % use 8-bit T1 fonts
\usepackage{hyperref}       % hyperlinks
\usepackage{url}            % simple URL typesetting
\usepackage{booktabs}       % professional-quality tables
\usepackage{amsfonts}       % blackboard math symbols
\usepackage{nicefrac}       % compact symbols for 1/2, etc.
\usepackage{microtype}      % microtypography
\usepackage{xcolor}         % colors
\usepackage{amsmath, amssymb, amsfonts}
\usepackage{graphicx}
\usepackage{multicol,multirow}
\usepackage{subcaption}
\usepackage{algorithm}
\usepackage{algpseudocode}
\usepackage{tabularx}

\usepackage{amsthm}
\newtheorem{theorem}{Theorem}

\newtheorem{lemma}[theorem]{Lemma}

\newtheorem{corollary}[theorem]{Corollary}

\title{Fused Gromov–Wasserstein Distance with \\Feature Selection}

\author{%
  Harlin Lee \\
  School of Data Science and Society\\
  University of North Carolina at Chapel Hill\\
  \texttt{harlin@unc.edu} \\
  \And
  Ying Yu\thanks{Work initiated while at University of North Carolina at Chapel Hill.} \\
  Division of Computational Social Science \\
  Chinese University of Hong Kong, Shenzhen \\
  \texttt{yyu@cuhk.edu.cn} \\
  \AND
  Mingxin Li \\
  Courant Institute \\
  New York University \\
  \texttt{ml6223@nyu.edu} \\
  \And
  Ranthony A. Clark \\
  Department of Mathematics \\
  Duke University \\
  \texttt{ranthony.clark@duke.edu} \\
}

\begin{document}

\maketitle

\begin{abstract}
Fused Gromov–Wasserstein (FGW) distances provide a principled framework for comparing objects by jointly aligning structure and node features. However, existing FGW formulations treat all features uniformly, which limits interpretability and robustness in high‑dimensional settings where many features may be irrelevant or noisy.
We introduce FGW distances with feature selection, which incorporate adaptive feature suppression weights into the FGW objective to selectively downweight or suppress differentiating features during alignment. We propose two approaches: (1) regularized FGW with Lasso and Ridge penalties, and (2) FGW with simplex-constrained weights, including groupwise extensions.
We analyze the resulting models and establish their key theoretical properties, including bounds relative to classical FGW and Gromov–Wasserstein distances, and metric behavior. 
An efficient alternating minimization algorithm is developed. Experiments illustrate how feature suppression enhances interpretability and reveals task‑relevant structure, with a special application to computational redistricting.
\end{abstract}

\section{Introduction}

The Wasserstein distance and optimal transport (OT)~\cite{gabriel2019computational,villani2009optimal} provide a principled framework for comparing probability measures and have become central tools in machine learning~\citep{courty2017joint,arjovsky2017wasserstein,banerjee2025efficient}. The Gromov-Wasserstein (GW) distance~\cite{memoli2011gromov} extends the Wasserstein distance to settings where datasets have distinct intrinsic geometries, enabling structure-aware comparison without requiring a shared ambient space. The Fused Gromov-Wasserstein (FGW) distance~\cite{pmlr-v97-titouan19a,vayer2020fused} further integrates structural information with feature dissimilarity, yielding a joint geometric--feature alignment distance. GW and FGW have been particularly effective for comparing graphs and shapes~\cite{peyre2016gromov,xu2019gromov,brogat2022learning,ma2023fused}. 

Despite its versatility, existing FGW formulations treat all feature dimensions identically. The feature contribution to the objective is uniformly aggregated, and no mechanism exists to adapt or select features. This limitation is significant in high-dimensional settings, where many features are irrelevant, redundant, or noisy. In such cases, automatically identifying differentiating features, those along which two objects most strongly disagree, would improve both interpretability and performance, as is well established in compressed sensing \citep{donoho2006compressed}. However, to date, feature selection has not been incorporated into the FGW framework.

We introduce \textbf{feature-selected FGW (fsFGW)}\footnote{We use \textit{feature selection} to describe the goal of identifying features of interest. Technically, this is realized through \textit{feature suppression}, whereby each feature is assigned a nonnegative weight that modulates its contribution to the FGW cost.} by assigning a \textit{feature suppression weight}  $w_r \in [0,1]$ to each feature $r$ such that $w_r=1$ removes feature $r$ from the transport cost entirely. Features with large suppression weights are \emph{differentiating}--suppressing them allows the transport plan to ignore mismatching nodes on those features to reduce the overall cost. Without additional constraints, this formulation admits a trivial solution in which all weights collapse to one, suppressing every feature.

To avoid this collapse, we develop two strategies: 

\textbf{(1) fsFGW with weight penalty}~(Section~\ref{sec:regularized}), where a penalty $\lambda R(\mathbf{w})$ such as Lasso~\cite{tibshirani1996regression} and Ridge~\cite{hoerl1970Ridge} is added. We show that Lasso yields binary weights, while Ridge yields continuous weights.

\textbf{(2) fsFGW with simplex-constrained weights}~(Section~\ref{sec:norm}), where $\mathbf{w}$ is constrained to the probability simplex. For a fixed transport plan, suppression concentrates on the single most differentiating feature. We study a groupwise extension in which features within a group share a common suppression weight, yielding a hard group selection.

We establish key theoretical properties of fsFGW, including its metric behavior, in Section~\ref{sec:theory}. Both strategies are unified under a single alternating minimization algorithm~(Section~\ref{sec:alg}), differing only in the projection step of the weight update. Experiments on synthetic and real datasets (Section~\ref{sec:experiments}) illustrate the proposed methods, with a special application to computational redistricting~(Section~\ref{sec:redistricting}).

\section{Background and Motivation: Fused Gromov-Wasserstein (FGW)}
We introduce notation, describe the classical Fused Gromov-Wasserstein distance, and motivate the need for weight regularization in feature-selected FGW. 

\subsection{Classical FGW Distance} \label{sec:fgw_distance}
FGW considers two finite, weighted, structured datasets
$\mathcal{X} =
\big(\{1,\dots,n\}, \mathbf{C}^X, \mathbf{a}, \mathbf{X}\big)$, $\mathcal{Y}=
\big(\{1,\dots,m\}, \mathbf{C}^Y, \mathbf{b}, \mathbf{Y}\big),
$
where $\mathbf{C}^X \in \mathbb{R}^{n\times n}$, $\mathbf{C}^Y \in
\mathbb{R}^{m\times m}$ encode structural dissimilarities,
$\mathbf{a} \in \Delta^{n-1}$, $\mathbf{b} \in \Delta^{m-1}$ are
probability vectors, and $\mathbf{X} = (x_i)_{i=1}^n$,
$\mathbf{Y} = (y_j)_{j=1}^m$ are node features in $\mathbb{R}^d$.
Let
$
U(\mathbf{a},\mathbf{b}) :=
\Big\{\mathbf{T} \in \mathbb{R}_+^{n\times m} :
\mathbf{T}\mathbf{1}_m = \mathbf{a},\
\mathbf{T}^\top\mathbf{1}_n = \mathbf{b}\Big\}
$
be the transport polytope. For $q \ge 1$, the Gromov--Wasserstein (GW) term~\cite{memoli2011gromov} in the FGW objectiv is:
\begin{equation}
\mathrm{GW}(\mathbf{T})
:=
\sum_{i,i'=1}^n \sum_{j,j'=1}^m
\left|C^X_{ii'} - C^Y_{jj'}\right|^q\,T_{ij}\,T_{i'j'}\; \ge0.
\end{equation}
For any $\mathbf{T}\in U(\mathbf{a},\mathbf{b})$ and $q \ge 1$, define the $r$th \emph{feature score}:
\begin{equation}
s_r(\mathbf{T}) := \sum_{i,j} T_{ij}\,|x_{ir}-y_{jr}|^q\;\ge 0,
\end{equation}
which measures the dissimilarity contributed by feature $r \in [1,d]$ under $\mathbf{T}$. The classical FGW distance~\cite{vayer2020fused} is then for $\alpha\in[0,1]$:
\begin{equation}
\label{eq:fgw}
\mathrm{FGW}^\star = \left(\min_{\mathbf{T}\in U(\mathbf{a},\mathbf{b})}
\quad
\left((1-\alpha)\sum_{r=1}^d s_r(\mathbf{T})
+
\alpha\,\mathrm{GW}(\mathbf{T})\right)^{p}\quad\right)^{1/p}.
\end{equation}
We set $p=1$ which corresponds to the standard FGW objective without outer exponentiation and simplifies our analysis and optimization.

\subsection{Why Feature Suppression Requires Weight Regularization}
\label{sec:unified}

We assign each feature $r$ a suppression weight $w_r \in [0,1]$. Setting $w_r = 1$ removes feature $r$ from the transport cost entirely, while $w_r = 0$ retains it fully. Features with $w_r$ close to $1$ are \emph{differentiating}, the transport plan ignores mismatching nodes on those features to drive the total transport cost down. The FGW problem with suppression weights $\mathbf{w}=(w_r)\in[0,1]^d$ is
\begin{equation}
\label{eq:fgw-weighted}
\min_{\mathbf{T}\in U(\mathbf{a},\mathbf{b}),\;\mathbf{w}\in[0,1]^d}
\quad
(1-\alpha)\sum_{r=1}^d (1-w_r)\,s_r(\mathbf{T})
+
\alpha\,\mathrm{GW}(\mathbf{T}).
\end{equation}
The connection to \eqref{eq:fgw} is immediate--setting $\mathbf{w}=\mathbf{0}$ recovers the classical FGW objective.

\begin{lemma}[Trivial solution]
\label{lem:trivial}
Problem~\eqref{eq:fgw-weighted} has a trivial minimizer at $\mathbf{w}^\star = \mathbf{1}$ for any $\mathbf{T}$.
\end{lemma}
This follows from non-negativity of each term.
To avoid this collapse, we introduce two strategies discussed in Section~\ref{sec:fsfgw}.

\section{Feature-Selected FGW (fsFGW) with Suppression Weight Regularization} \label{sec:fsfgw}
We describe two approaches to feature-selected FGW with regularization to address the degeneracy identified in Lemma\ref{lem:trivial}. One approach adds weight penalty such as Lasso and Ridge, while the second approach constrains the weights to the probability simplex.

\subsection{fsFGW with Weight Penalty}
\label{sec:regularized}

We first penalize the suppression weights via a regularizer $R(\mathbf{w})$, which discourages the trivial all-ones solution while enabling adaptive feature selection. Adding a regularizer $R:\mathbb{R}^d\to\mathbb{R}_+$ gives
\begin{equation}
\label{eq:fgw-general}
\min_{\mathbf{T}\in U(\mathbf{a},\mathbf{b}),\;\mathbf{w}\in[0,1]^d}
\quad
(1-\alpha)\sum_{r=1}^d (1-w_r)\,s_r(\mathbf{T})
+
\alpha\,\mathrm{GW}(\mathbf{T})
+
\lambda R(\mathbf{w}),
\qquad \lambda > 0,
\end{equation}
where the trivial solution is no longer optimal. For fixed $\mathbf{T}$, the $\mathbf{w}$-subproblem is
\begin{equation}
\label{eq:w-subproblem}
\min_{\mathbf{w}\in[0,1]^d}
\quad
-(1-\alpha)\sum_{r=1}^d w_r\,s_r(\mathbf{T})
+
\lambda R(\mathbf{w}).
\end{equation}

\textbf{Lasso.} Because the weights are non-negative, subproblem~\eqref{eq:w-subproblem} with $R(\mathbf{w}) = \|\mathbf{w}\|_1$ decouples into $d$ independent linear programs:
\begin{equation}
\label{eq:coord-lasso}
\min_{w_r \in [0,1]}
\quad
\big(\lambda - (1-\alpha)\,s_r(\mathbf{T})\big)\,w_r,
\qquad r = 1,\dots,d.
\end{equation}

For fixed $\mathbf{T}$, the solution to \eqref{eq:coord-lasso} is
\begin{equation}
w_r^\star =
\begin{cases}
1, & \text{if } (1-\alpha)\,s_r(\mathbf{T}) > \lambda,\\[4pt]
0, & \text{if } (1-\alpha)\,s_r(\mathbf{T}) < \lambda,
\end{cases}
\end{equation}
with any $w_r^\star\in[0,1]$ optimal when $(1-\alpha)s_r(\mathbf{T})=\lambda$.

\textbf{Ridge.} With $R(\mathbf{w}) = \frac{1}{2}\|\mathbf{w}\|_2^2$, subproblem~\eqref{eq:w-subproblem} decouples into $d$ quadratics:
\begin{equation}
\label{eq:coord-ridge}
\min_{w_r\in[0,1]}
\quad
-\,(1-\alpha)\,s_r(\mathbf{T})\,w_r
+
\frac{\lambda}{2}\,w_r^2,
\qquad r=1,\dots,d.
\end{equation}

For fixed $\mathbf{T}$, the solution to \eqref{eq:coord-ridge} is
\begin{equation}
w_r^\star
=
\min\!\left\{1,\;\frac{(1-\alpha)\,s_r(\mathbf{T})}{\lambda}\right\}.
\end{equation}

\textbf{Summary.} In Lasso, feature $r$ is fully suppressed when its score exceeds $\lambda/(1-\alpha)$ and fully retained otherwise, producing \textit{hard suppression}. On the other hand, Ridge produces \emph{smooth suppression}: $w_r^\star$ increases continuously with $s_r(\mathbf{T})$, saturating at 1. The transition occurs at the same point, when $s_r(\mathbf{T}) \ge \lambda/(1-\alpha)$, and higher $\lambda$ and $\alpha$ lead to sparser or smaller weights in both cases.

\subsection{fsFGW with Simplex-Constrained Weights}
\label{sec:norm}

Another alternative is to constrain $\mathbf{w}$ to the probability simplex $\Delta^{d-1} = \{\mathbf{w}\in\mathbb{R}_+^d : \|\mathbf{w}\|_1 = 1\}$, preventing the all-ones vector. The simplex-constrained FGW problem is
\begin{equation}
\label{eq:simplex-constrained}
\min_{\mathbf{T}\in U(\mathbf{a},\mathbf{b}),\;\mathbf{w}\in\Delta^{d-1}}
\quad
(1-\alpha)\sum_{r=1}^d (1-w_r)\,s_r(\mathbf{T})
+
\alpha\,\mathrm{GW}(\mathbf{T}).
\end{equation}
For fixed $\mathbf{T}$, the $\mathbf{w}$-subproblem is
\begin{equation}
\label{eq:norm-w-subproblem}
\max_{\mathbf{w}\in\Delta^{d-1}}
\quad
\sum_{r=1}^d w_r\,s_r(\mathbf{T}),
\end{equation}
which is a linear function maximized over the simplex, attained at
\begin{equation}
r^\star \in \arg\max_{r=1,\dots,d}\, s_r(\mathbf{T}),
\qquad
\mathbf{w}^\star = \mathbf{e}_{r^\star}.
\end{equation}

\textbf{Groupwise Extension.} In many applications, features naturally form groups, and it is desirable to suppress or retain an entire group together. Let $\{\mathcal{G}_1,\dots,\mathcal{G}_G\}$ be a partition of $\{1,\dots,d\}$. We assign a common suppression weight $w_i$ to every feature in group $\mathcal{G}_i$, so $\mathbf{w}=(w_1,\dots,w_G)\in\Delta^{G-1}$ with induced per-feature suppression weights
$ w_r := w_i$ whenever $r\in\mathcal{G}_i.$
The non-group case is recovered by taking $G=d$ singleton groups. The groupwise simplex-constrained FGW problem is
\begin{equation}
\label{eq:groupwise-simplex}
\min_{\mathbf{T}\in U(\mathbf{a},\mathbf{b}),\;
\mathbf{w}\in\Delta^{G-1}}
\quad
(1-\alpha)\sum_{i=1}^G (1-w_i)\frac{1}{|\mathcal{G}_i|}\sum_{r\in\mathcal{G}_i} s_r(\mathbf{T})
+
\alpha\,\mathrm{GW}(\mathbf{T}).
\end{equation}
In this case, the minimum for the $\mathbf{w}$-subproblem is attained at
\begin{equation}
i^\star \in \arg\max_{i=1,\dots,G}\, \frac{1}{|\mathcal{G}_i|}\sum_{r\in\mathcal{G}_i} s_r(\mathbf{T}),
\qquad
\mathbf{w}^\star = \mathbf{e}_{i^\star}.
\end{equation}

\textbf{Summary.} For fsFGW with simplex-constrained weights, the single most dissimilar feature $r^\star$ under $\mathbf{T}$, or group of features $\mathcal{G}_{i^\star}$ for the groupwise extension, is fully suppressed and all others are retained.

\section{Theoretical Properties of the fsFGW Distance}
\label{sec:theory}
The two versions of fsFGW in Sections~\ref{sec:regularized} and \ref{sec:norm} share a common structure, with constraint set $\mathcal{W}$ and regularizer $R$ given by:
\begin{itemize}
    \item \textbf{Weight penalty}: $\mathcal{W} = [0,1]^d$,\;
    $R = \|\mathbf{w}\|_1$ (Lasso) or $R = \tfrac{1}{2}\|\mathbf{w}\|_2^2$
    (Ridge), $\lambda > 0$.
    \item \textbf{Simplex-constrained weights}: $\mathcal{W} = \Delta^{d-1}$
    (or $\Delta^{G-1}$ for the groupwise case),  $\lambda, R = 0$.
\end{itemize}
The full joint problem over $(\mathbf{T}, \mathbf{w})$ is then
\begin{equation}
\label{eq:fgw-unified}
\text{fsFGW}^\star = \min_{\mathbf{T}\in U(\mathbf{a},\mathbf{b}),\;\mathbf{w}\in\mathcal{W}}
\quad
(1-\alpha)\sum_{r=1}^d (1-w_r)\,s_r(\mathbf{T})
+
\alpha\,\mathrm{GW}(\mathbf{T})
+
\lambda R(\mathbf{w}).
\end{equation}
As in GW and FGW, the feature-suppressed FGW problem extends naturally to general metric measure spaces. Appendix \ref{app:proofs} contains this definition and all proofs from this section.

We start by establishing existence and comparing fsFGW distance to FGW and GW distances.
\begin{theorem}[Existence]
\label{thm:existence} 
Let $\mathcal{X}$ and $\mathcal{Y}$ be two finite metric measure spaces defined as in Section~\ref{sec:fgw_distance}. Then the optimization problem \eqref{eq:fgw-unified} admits at least one minimizer $(\mathbf{T}^\star, \mathbf{w}^\star)$.
\end{theorem}

The proof follows from standard compactness and continuity arguments used for GW \cite{memoli2011gromov} and FGW \cite{vayer2020fused} to the joint variable $(\mathbf{T}, \mathbf{w}^\star)$. 

\begin{theorem}[Bounds]
\label{thm:bounds}
Let $\mathrm{GW}^\star, \mathrm{fsFGW}^\star$ and $\mathrm{FGW}^\star$ denote the optimal GW, feature-selected FGW and classical FGW distances between $\mathcal{X}$ and $\mathcal{Y}$. Then the optimal fsFGW distance satisfies
\begin{equation}
    \alpha \cdot \mathrm{GW}^\star 
    \;\leq\; 
    \mathrm{fsFGW}^\star 
    \;\leq\; 
    \mathrm{FGW}^\star
\end{equation}
for all four modes: Lasso, Ridge, Simplex, and Groupwise Simplex.
\end{theorem}

These bounds enable Corollary \ref{cor:convergence} on convergence of finite samples in Appendix \ref{app:proofs} via a sandwiching argument. Finally, we conclude with a discussion on metric properties of fsFGW.

\begin{theorem}[Metric]
\label{thm:metric}
fsFGW distance enjoys several properties. For fixed $\lambda>0$,
\begin{itemize}
    \item Positivity and Symmetry: Satisfied for all four modes.
    \item Identity of indiscernibles: Satisfied for Lasso and Ridge.
    \item Triangle inequality: For Lasso and Ridge, fsFGW satisfies triangle inequality for $q=1$. When $q>1$, it satisfies a relaxed triangle inequality by a factor $2^{q-1}$.
\end{itemize}

\end{theorem}
Theorem~\ref{thm:metric} states fsFGW with Lasso or Ridge is a metric when $q=1$ (with $\mathcal{W} = [0,1]^d, \text{ fixed } \lambda > 0)$, and semi-metric when $q>1$. The proof uses FGW results from \cite{vayer2020fused}, closed-form formulations of $\mathbf{w}^\star$, and non-negativity of $\mathbf{w}^\star$ and terms in \eqref{eq:fgw-unified}. 
For simplex-based variants $(\mathcal{W} = \Delta^{d-1}, \Delta^{G-1})$, the triangle inequality and identity of indiscernibles may fail. 

\section{Alternating Minimization Algorithm}
\label{sec:alg}

We optimize all four modes with the same alternating minimization structure (c.f. Algorithm~\ref{alg:sfgw} in Appendix~\ref{app:convergence}): alternate between a transport update for fixed $\mathbf{w}$ and a weight update for fixed $\mathbf{T}$. 

\textbf{Transport update.}
For fixed $\mathbf{w}$, the $\mathbf{T}$-subproblem is a classical FGW problem with weighted feature cost $\sum_r(1-w_r)s_r(\mathbf{T})$. This is a non-convex problem that can be solved to a stationary point via conditional gradient~\cite{vayer2020fused}, a certifiable lower bound and, when tight, a global minimizer, via semidefinite programming~\cite{chen2024semidefinite}, or to a fixed point set via Bregman alternating projected gradient~\cite{li2023convergent}.

\textbf{Weight update.}
For fixed $\mathbf{T}$, $\textbf{w}^\star$ is available in closed form in all four cases, as established in Sections~\ref{sec:regularized} and~\ref{sec:norm}. We initialize $\textbf{w}=\textbf{0}$, so that the initial transport plan, $\mathbf{T}_0$, corresponds to classical FGW. 

\cite{lacoste2016convergence} gives the following rate for Ridge, which can be reformulated into a non-convex, continuously differentiable function on $\mathbf{T}$ as in \eqref{eq:only_T}. In practice, most runs converged in 2-8 steps for all modes. 
\begin{lemma}[Convergence]
    Conditional gradient converges to a stationary point of \eqref{eq:only_T} at $O(1/\sqrt{t})$ for fsFGW with Ridge regularization.
\end{lemma}

\textbf{$\lambda$ selection in Ridge and Lasso.} We set $\lambda$ using the initial transport plan $\mathbf{T}_0$. Specifically,
\begin{equation}
\lambda = (1-\alpha)\cdot \mathrm{Quantile}_{1-f}\!\left(s_1(\mathbf{T}_0),\dots,s_d(\mathbf{T}_0)\right),
\end{equation}
where $f\in(0,1)$ is a user-chosen \emph{suppression fraction}. This choice suppresses approximately $f\cdot d$ features under $\mathbf{T}_0$. Since $\mathbf{T}_0$ is already computed at the first iteration of alternating minimization, this calibration incurs no additional FGW matching cost. If $\lambda$ is specified by the user, this step is omitted. Since Ridge suppresses features gradually rather than all-at-once, a lower $f$ is recommended compared to Lasso to achieve a similar effect as seen in Section~\ref{sec:synthetic}.

\section{Numerical Experiments}
\label{sec:experiments}

We first evaluate fsFGW on synthetic graphs to verify that learned weights identify differentiating features. We then compare fsFGW to GW and FGW on graph classification and clustering benchmark datasets. $\alpha=0.5, q=2$ unless otherwise stated.

\subsection{Synthetic Data}
\label{sec:synthetic}

We generate two random geometric graphs $X, Y$ on $n=40$ nodes amd compute normalized geodesic distances matrices $C_1, C_2 \in [0,1]^{n \times n}$. Node features are independent of structure, of $d$ features total, $k$ are \emph{differentiating} and $d-k$ are \emph{shared}. $k$ features have distribution $\mathcal{N}(0,1)$ in $X$ and $\mathcal{N}(\delta,1)$ in $Y$, while the rest are $\mathcal{N}(0,1)$ in both. See details in Appendix \ref{app:synthetic}.

\begin{figure}[htp]
    \centering
    \includegraphics[width=\linewidth]{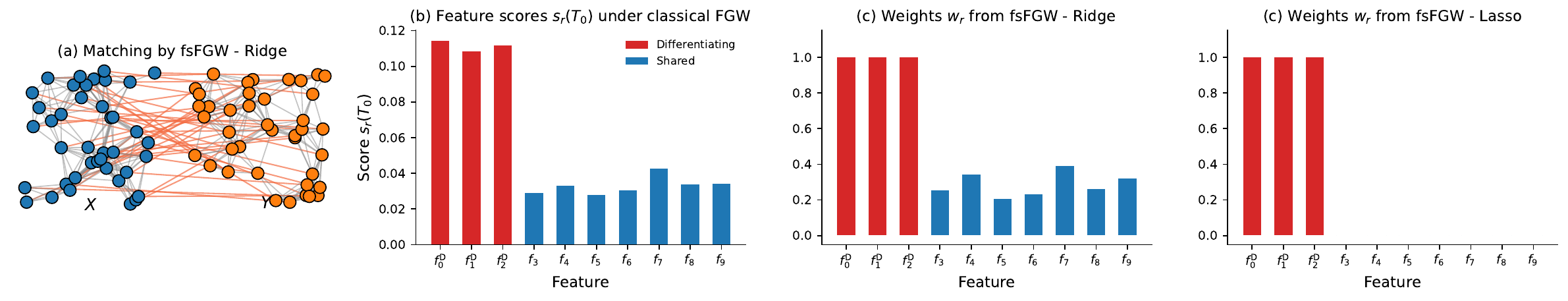}
    \caption{(a) Graphs matched by fsFGW - Ridge. (b) Feature scores $s_r(\textbf{T}_0)$ under classical FGW: differentiating features (red) score $3$--$4\times$ higher than 
    shared features (blue). (c) and (d) fsFGW with Ridge ($f=0.2$) and Lasso ($f=0.3$) successfully suppresses the differentiating features.}
    \label{fig:synthetic_main}
\end{figure}

\textbf{fsFGW can identify differentiating features.}
We start with a low-dimensional instance ($d=10$, $k=3$, $\delta=2.0$) of the toy graphs.
Figure~\ref{fig:synthetic_main}b shows the feature scores $s_r(\textbf{T}_0)$ under the classical FGW plan, obtained with $\textbf{w}=\mathbf{0}$. Differentiating features score $3$--$4\times$ higher than shared features, confirming that the signal is visible to the transport plan before any suppression is applied. Figure~\ref{fig:synthetic_main}c and d shows the recovered suppression weights for Ridge ($f=0.2$) and Lasso ($f=0.3$). Both modes correctly assign high suppression to differentiating features and low suppression to shared ones. Appendix \ref{app:synthetic} shows graph matching (Figure~\ref{fig:matched_graphs}) and learned feature weights (Figure~\ref{fig:exp2_baseline}) of all four modes of fsFGW. Figure~\ref{fig:agreement_heatmap} shows that the four modes of fsFGW generate substantially distinct transport plans from each other, as well as from GW and FGW.

\begin{figure}[t]
\begin{subfigure}[b]{0.5\linewidth}
    \includegraphics[width=\linewidth]{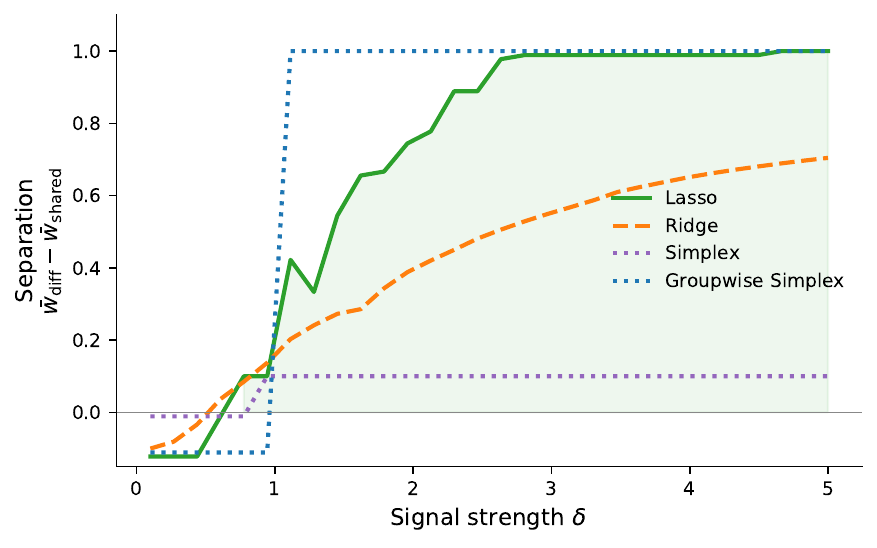}
    \subcaption{Signal strength $\delta$ sweep. Lasso ($f=0.1$) and Groupwise simplex achieves perfect separation above $\delta \approx 3$ and 1, exhibiting a sharp phase transition.}
    \label{fig:exp4_signal}
\end{subfigure}\hfill
\begin{subfigure}[b]{0.35\linewidth}
         \includegraphics[width=\textwidth]{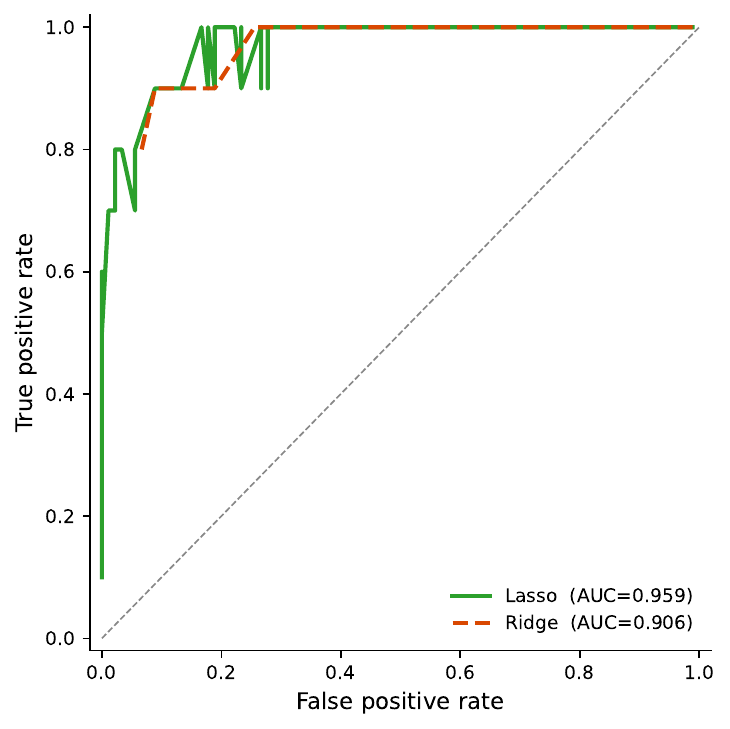}
         \subcaption{ROC curves sweeping suppression fraction $f$. Each point is a separate model run at a different $\lambda$.}
         \label{fig:exp3a_roc_sweep}
\end{subfigure}
\caption{fsFGW performance on toy graphs ($d=100, k=10$) are robust to $\delta$ and $f$ choice.}
\end{figure}

\textbf{Dependence on $\delta$ and $\lambda$.}
For this experiment, we generate a high-dimensional instance ($d=100$, $k=10$) of the toy graphs and sweep the mean shift $\delta \in [0.1,5.0]$. Figure~\ref{fig:exp4_signal} measures the  separation $\bar{w}_{\mathrm{diff}} - \bar{w}_{\mathrm{shared}}$ as a function of $\delta$, where separation $=1$ corresponds to perfect identification. Lasso ($f=0.1$) and Group simplex (given correct grouping) achieves perfect separation above $\delta \approx 3$ and 1, exhibiting a sharp phase transition. Ridge ($f=0.05$) improves continuously but does not reach 1 due to partial suppression of shared features. Simplex is bounded above by $1/k=0.1$ since it suppresses only one feature out of $k=10$. 

Finally, we examine how the choice of $\lambda$ affects recovery quality at $\delta=2.0$. We sweep $f \in (0,1)$ and threshold recovered weights $w_r$ at $0.5$ to obtain binary predictions. Figure~\ref{fig:exp3a_roc_sweep} shows ROC curves for Lasso (AUC $=0.959$) and Ridge (AUC $=0.906$).

\subsection{Graph Classification and Clustering Benchmark Datasets}
\label{sec:benchmark}
\begin{table}[t]
\centering
\resizebox{\textwidth}{!}{
\begin{tabular}{lllccccc}
\toprule
Dataset & Reg & Metric & GW & FGW & $f=0.3$ & $f=0.5$ & $f=0.7$ \\
\midrule
\multirow{8}{2.5cm}{FRANKENSTEIN ($d=780$)} & \multirow{4}{*}{Lasso} & Acc (\%) & 54.5$_{\pm 6.5}$ & 51.0$_{\pm 8.6}$ & 53.0$_{\pm 9.3}$ & 51.5$_{\pm 13.6}$ & \textbf{59.0$_{\pm 9.2}$} \\
 &  & F1 (\%) & 52.8$_{\pm 6.1}$ & 45.0$_{\pm 8.7}$ & 46.3$_{\pm 8.9}$ & 50.0$_{\pm 13.2}$ & \textbf{57.1$_{\pm 9.0}$} \\
 &  & NMI & \textbf{0.017$_{\pm 0.000}$} & 0.013$_{\pm 0.000}$ & 0.000$_{\pm 0.000}$ & 0.000$_{\pm 0.000}$ & \textbf{0.017$_{\pm 0.000}$} \\
 &  & ARI & 0.016$_{\pm 0.000}$ & \textbf{0.017$_{\pm 0.000}$} & -0.005$_{\pm 0.000}$ & -0.003$_{\pm 0.001}$ & 0.016$_{\pm 0.000}$ \\
\cmidrule{2-8}
 & \multirow{4}{*}{Ridge} & Acc (\%) & 54.5$_{\pm 6.5}$ & 51.0$_{\pm 8.6}$ & 56.5$_{\pm 6.7}$ & 53.5$_{\pm 13.2}$ & \textbf{59.0$_{\pm 9.2}$} \\
 &  & F1 (\%) & 52.8$_{\pm 6.1}$ & 45.0$_{\pm 8.7}$ & 53.5$_{\pm 7.3}$ & 52.0$_{\pm 12.9}$ & \textbf{57.1$_{\pm 9.0}$} \\
 &  & NMI & \textbf{0.017$_{\pm 0.000}$} & 0.013$_{\pm 0.000}$ & 0.001$_{\pm 0.000}$ & 0.000$_{\pm 0.000}$ & \textbf{0.017$_{\pm 0.000}$} \\
 &  & ARI & 0.016$_{\pm 0.000}$ & \textbf{0.017$_{\pm 0.000}$} & -0.001$_{\pm 0.000}$ & -0.002$_{\pm 0.001}$ & 0.016$_{\pm 0.000}$ \\
\midrule
\multirow{8}{2.5cm}{PROTEINS\_full ($d=32$)} & \multirow{4}{*}{Lasso} & Acc (\%) & 70.5$_{\pm 9.6}$ & 62.0$_{\pm 7.5}$ & \textbf{71.5$_{\pm 8.1}$} & 67.5$_{\pm 11.5}$ & 62.5$_{\pm 10.1}$ \\
 &  & F1 (\%) & 68.8$_{\pm 10.5}$ & 55.9$_{\pm 11.0}$ & \textbf{69.6$_{\pm 8.6}$} & 65.5$_{\pm 11.7}$ & 60.9$_{\pm 10.4}$ \\
 &  & NMI & 0.068$_{\pm 0.000}$ & \textbf{0.096$_{\pm 0.002}$} & 0.029$_{\pm 0.000}$ & 0.025$_{\pm 0.000}$ & 0.007$_{\pm 0.000}$ \\
 &  & ARI & 0.056$_{\pm 0.000}$ & \textbf{0.104$_{\pm 0.000}$} & 0.026$_{\pm 0.000}$ & 0.022$_{\pm 0.000}$ & 0.007$_{\pm 0.000}$ \\
\cmidrule{2-8}
 & \multirow{4}{*}{Ridge} & Acc (\%) & \textbf{70.5$_{\pm 9.6}$} & 62.0$_{\pm 7.5}$ & 60.0$_{\pm 14.0}$ & 63.5$_{\pm 10.5}$ & 53.0$_{\pm 9.3}$ \\
 &  & F1 (\%) & \textbf{68.8$_{\pm 10.5}$} & 55.9$_{\pm 11.0}$ & 58.3$_{\pm 13.8}$ & 61.7$_{\pm 11.3}$ & 50.7$_{\pm 10.2}$ \\
 &  & NMI & 0.068$_{\pm 0.000}$ & 0.096$_{\pm 0.002}$ & 0.061$_{\pm 0.001}$ & 0.020$_{\pm 0.000}$ & \textbf{0.100$_{\pm 0.004}$} \\
 &  & ARI & 0.056$_{\pm 0.000}$ & 0.104$_{\pm 0.000}$ & 0.062$_{\pm 0.002}$ & 0.019$_{\pm 0.000}$ & \textbf{0.108$_{\pm 0.003}$} \\
\midrule
\multirow{8}{2.5cm}{ogbg-molbace ($d=9$)} & \multirow{4}{*}{Lasso} & Acc (\%) & 76.5$_{\pm 8.4}$ & \textbf{77.0$_{\pm 5.6}$} & \textbf{77.0$_{\pm 4.6}$} & 71.0$_{\pm 5.4}$ & 65.0$_{\pm 8.9}$ \\
 &  & F1 (\%) & \textbf{69.2$_{\pm 10.4}$} & 66.4$_{\pm 9.7}$ & 68.5$_{\pm 6.5}$ & 61.0$_{\pm 10.5}$ & 55.5$_{\pm 8.8}$ \\
 &  & NMI & 0.000$_{\pm 0.001}$ & \textbf{0.023$_{\pm 0.006}$} & 0.003$_{\pm 0.000}$ & 0.004$_{\pm 0.000}$ & 0.013$_{\pm 0.000}$ \\
 &  & ARI & 0.001$_{\pm 0.006}$ & -0.031$_{\pm 0.006}$ & 0.011$_{\pm 0.001}$ & 0.017$_{\pm 0.000}$ & \textbf{0.032$_{\pm 0.000}$} \\
\cmidrule{2-8}
 & \multirow{4}{*}{Ridge} & Acc (\%) & 76.5$_{\pm 8.4}$ & \textbf{77.0$_{\pm 5.6}$} & 63.0$_{\pm 7.8}$ & 68.0$_{\pm 6.0}$ & 69.0$_{\pm 5.8}$ \\
 &  & F1 (\%) & \textbf{69.2$_{\pm 10.4}$} & 66.4$_{\pm 9.7}$ & 53.1$_{\pm 8.8}$ & 58.1$_{\pm 10.3}$ & 61.0$_{\pm 8.8}$ \\
 &  & NMI & 0.000$_{\pm 0.001}$ & \textbf{0.023$_{\pm 0.006}$} & 0.018$_{\pm 0.002}$ & 0.007$_{\pm 0.002}$ & 0.002$_{\pm 0.000}$ \\
 &  & ARI & 0.001$_{\pm 0.006}$ & -0.031$_{\pm 0.006}$ & -0.029$_{\pm 0.001}$ & -0.025$_{\pm 0.003}$ & \textbf{0.010$_{\pm 0.002}$} \\
\bottomrule
\end{tabular}}
\caption{Classification (10-fold CV SVM: accuracy, macro F1) and clustering (10 runs of k-means: NMI, ARI) results. Bold indicates best per row. $f$ denotes suppression fraction.}
\label{tab:benchmark}
\end{table}

Now that we have a better understanding of how fsFGW works, we test it on several benchmark datasets. But first, we clarify that fsFGW, like GW and FGW, is not a classifier. Instead, it is a method for computing an interpretable pairwise distance between graphs that simultaneously identifies which node features drive the dissimilarity. This section is evaluates whether the resulting distance remains a useful dissimilarity that preserves the data structure at least as well as classical FGW and GW.  More experimental details and results are available in Appendix \ref{app:benchmark}.

We evaluate on three datasets. \textsc{Frankenstein}~\cite{orsini2015graph} contains 4{,}337 molecular graphs with 2 classes and high-dimensional node attributes ($d = 780$).  \textsc{Proteins-full}~\cite{borgwardt2005protein} consists of 1{,}113 protein graphs with 2 classes, where nodes represent secondary structure elements annotated with $d = 29$ features. \textsc{ogbg-molbace}~\cite{hu2020open}, consisting of 1{,}513 molecular graphs with 2 classes and $d = 9$ atom-level features (atomic number, chirality, degree, formal charge, number of hydrogens, number of radical electrons, hybridization, aromaticity, ring membership). For each dataset we subsample a stratified subset of 200 graphs and compute the full pairwise distance matrix. We compare GW ($\alpha=1$), FGW ($\alpha=0.5$, $\mathbf{w}=\mathbf{0}$), fsFGW (Lasso), and fsFGW (Ridge) across $f \in \{0.3, 0.5, 0.7\}$. Classification uses a 10-fold CV SVM with RBF kernel; clustering uses $k$-means averaged over 10 seeds. We report accuracy, F1, Normalized Mutual Information (NMI) and Adjusted Rand Index (ARI).

Table~\ref{tab:benchmark} reports results across all three datasets and suppression fractions. The primary takeaway is that fsFGW is \emph{competitive} with GW and classical FGW across all settings, while additionally providing interpretable feature weights.
On \textsc{Frankenstein} ($d=780$), classical FGW underperforms GW on all metrics, which suggests that the structure correlates better with the class labels compared to its features. Feature suppression at $f=0.7$ recovers GW performance.
On \textsc{Proteins-full} ($d=32$), fsFGW (Lasso) at $f=0.3$ matches GW on classification (71.5\% vs 70.5\%) and improves over classical FGW by 9.5\%. fsFGW (Ridge) at $f=0.7$ achieves the highest NMI (0.100) and ARI (0.108), surpassing both GW and FGW on clustering. Lasso is more robust across fractions, while Ridge is more sensitive: Ridge at $f=0.7$ achieves the best clustering but degrades on classification.

On \textsc{ogbg-molbace} ($d=9$), GW and FGW achieve comparable classification accuracy ($\sim$77\%), with fsFGW (Lasso) at $f=0.3$ matching this level. Notably, classical FGW yields negative ARI ($-0.031$), worse than random clustering, while fsFGW (Lasso) recovers positive ARI across all fractions, peaking at $0.032$ for $f=0.7$. Since \textsc{ogbg-molbace} comes with interpretable features, we plot mean suppression weights $\bar{w}_r$ per feature in Figure~\ref{fig:weights_molbace}. Across both Lasso and Ridge, \texttt{hybridization}, \texttt{is\_aromatic}, and \texttt{is\_in\_ring} consistently receive the highest suppression weights, identifying these as the most differentiating atom properties between molecule pairs. In contrast, \texttt{num\_radical\_e} is consistently retained, suggesting this is uniformly distributed across active and inactive compounds. 
This level of interpretability is not available from classical FGW.

\begin{figure}[t]
    \centering
    \includegraphics[width=\linewidth]{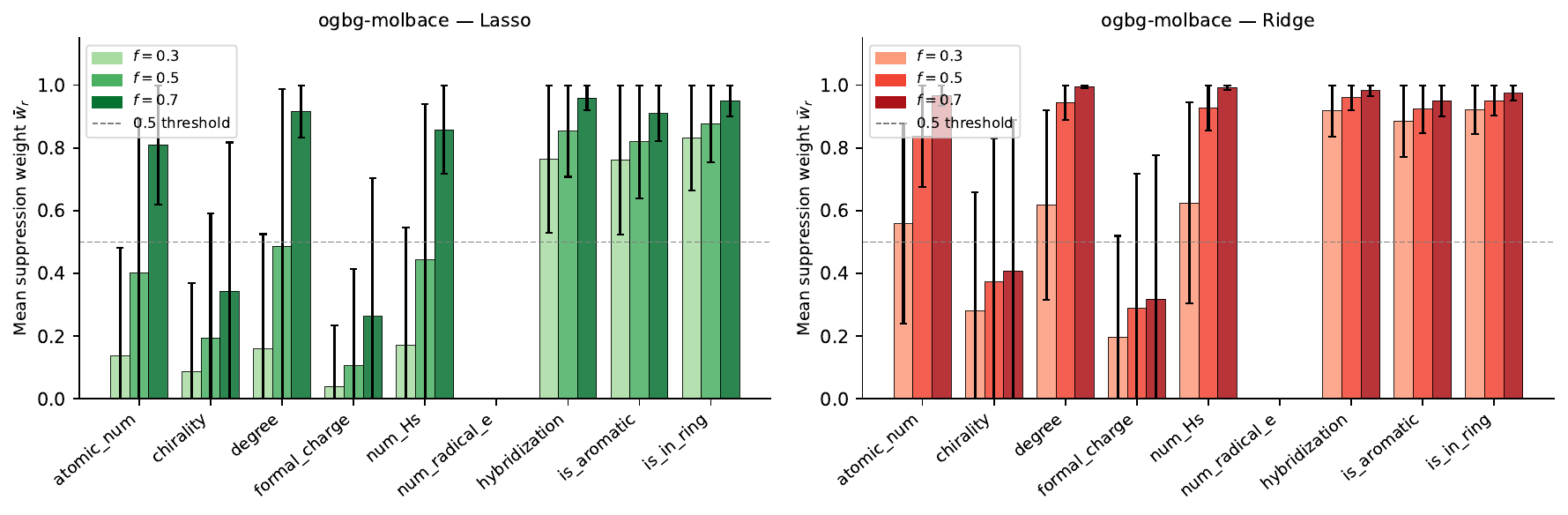}
    \caption{Mean suppression weights $\bar{w}_r$ per feature on \textsc{ogbg-molbace}, averaged over all graph pairs. \{\texttt{hybridization}, \texttt{is\_aromatic}, \texttt{is\_in\_ring}\} are most strongly suppressed across both Lasso and Ridge. Next are \{\texttt{atomic\_num}, \texttt{degree}, \texttt{num\_Hs}\}, then \{\texttt{chirality}, \texttt{formal\_charge}\}. \texttt{num\_radical\_e} is consistently retained.}
    \label{fig:weights_molbace}
\end{figure}

\section{Application to Computational Redistricting}
\label{sec:redistricting}

In the United States, a redistricting \textit{plan} partitions a state into electoral \textit{districts}, each consisting of precincts that elect representatives. Because district boundaries determine how votes translate into representation, redistricting plays a central role in electoral outcomes. This motivates computational redistricting~\cite{duchin2022political}, which models plans as graph partitioning problems.

We represent a district as a metric measure space on a graph $G=(V,E)$, where vertices are precincts, and edges encode geographical adjacency. Each district induces: (i) a structure matrix $\mathbf{C}$ capturing precinct distances on the graph, (ii) node features $\mathbf{X}$ encoding demographic and electoral data at precinct level, and (iii) a probability measure $\mathbf{a}$ over precincts (e.g. normalized population). Thus, each \textit{district} is represented in the FGW framework as $(\mathbf{C}, \mathbf{a}, \mathbf{X})$. To compare \textit{plans}, we match districts using linear sum assignment and aggregate district-level GW, FGW, or fsFGW distances into a plan-level distance as in \cite{clark2025generalized}. Details and the 29 features used are given in Appendix~\ref{app:redistricting}. 

Works that apply Wasserstein~\cite{abrishami2020geometry} and GW~\cite{clark2025generalized} distances to redistricting plans exist. To our knowledge, this is the first geometrically aware metric in computational redistricting that applies FGW as well as provides feature-level attribution of differences between redistricting plans, thereby identifying interpretable drivers of the differences between plans. 
 
We use this setting to evaluate three properties of fsFGW: 
(i) whether features alter plan similarity beyond geometry,
(ii) whether suppression captures meaningful distributed differences, and
(iii) whether differences localize to specific districts.
As a case study, we examine six redistricting plans in North Carolina between 2020 and 2025, a period which spans a court-imposed, bias-constrained plan (2022) and subsequent legislatively enacted plans, including a mid-cycle revision (2025). These maps (Figure~\ref{fig:ncmaps}) provide a natural testbed due to variation in geometry and electoral composition. 

\textbf{Plan similarity beyond geometry} We compute pairwise GW, FGW, and fsFGW distances between plans to compare how each method captures similarities between plans. Figure~\ref{fig:dendro_comparison} shows that GW, which captures only geometry, groups plans with similar spatial layouts regardless of political composition. Incorporating features via FGW reshapes this structure, separating the court-constrained plan (22ct) from its legislative counterpart (22) and drawing legislatively enacted plans (23, 25) closer together. Similarly, fsFGW with Lasso reiterates this pattern. See Figure~\ref{fig:dendrograms} for more results. 
\begin{figure}[ht]
    \centering
    
    \begin{subfigure}[t]{0.33\textwidth}
        \centering
        \includegraphics[width=\textwidth]{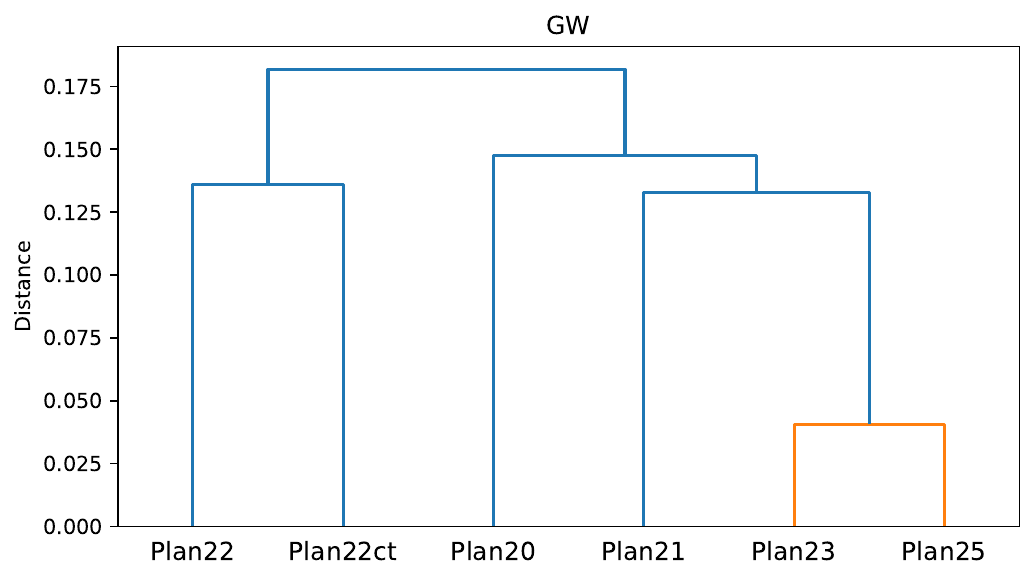}
        \label{fig:dendro_a}
    \end{subfigure}%
    \begin{subfigure}[t]{0.33\textwidth}
        \centering
        \includegraphics[width=\textwidth]{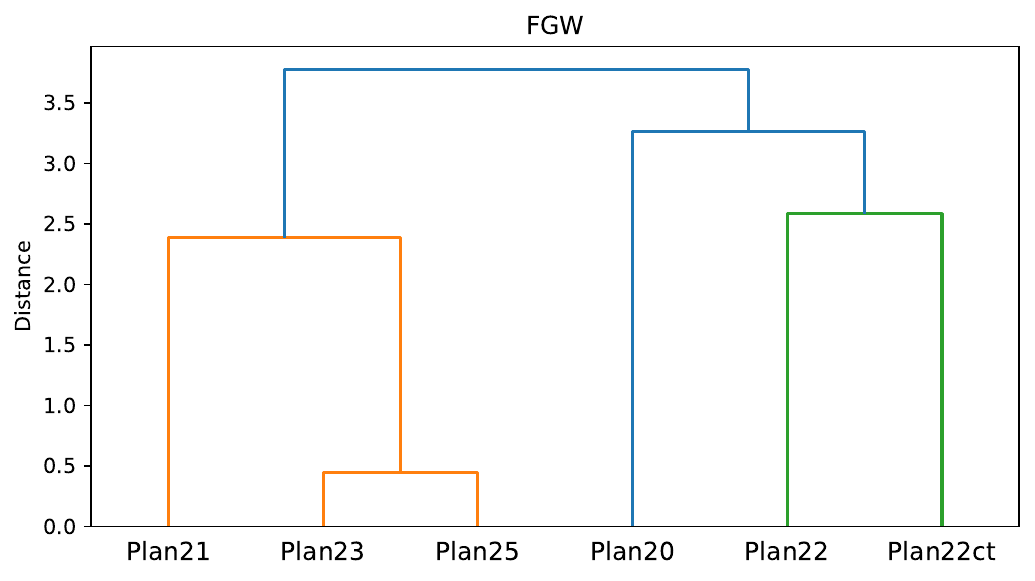}
        \label{fig:dendro_b}
    \end{subfigure}%
    \begin{subfigure}[t]{0.33\textwidth}
        \centering
        \includegraphics[width=\textwidth]{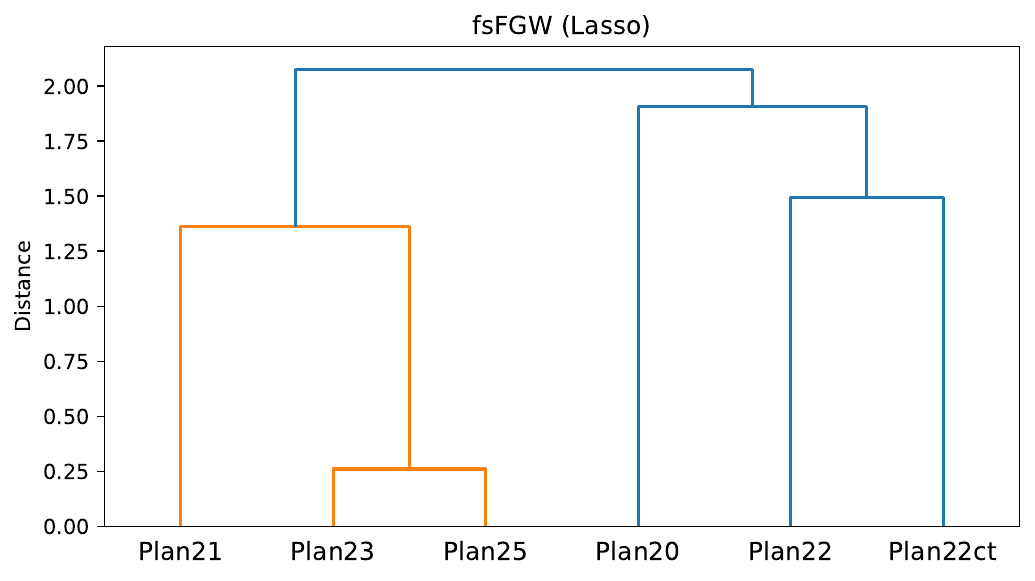}
        \label{fig:dendro_c}
    \end{subfigure}
    
    \caption{Hierarchical clustering of plans with GW, FGW, and fsFGW (Lasso). }
    \label{fig:dendro_comparison}
\end{figure}

\textbf{Feature-level changes across plans} To further understand how fsFGW responds to controlled structural changes, we examine the transition from Plan 22 (Senate Bill 745) to its court-modified version 22ct (\textit{Harper v. Hall}). This setting isolates judicially induced changes, and we visualize the fsFGW (Simplex) suppression weights in Figure~\ref{fig:mean_weights_22_22ct}; the other modes are in Figure~\ref{fig:ncweights22v22ct}. Note that the standard deviation error bar is large because Simplex selects only one feature per district-pair. The suppression weights indicate that the court modification redistributes demographic, political, and socioeconomic characteristics across districts—particularly housing type, racial composition (Black, BVAP), political alignment (G20Dem), and urbanization proxies—consistent with a shift toward reduced partisan bias and increased competitiveness. 

\begin{figure}[htp]
    \centering
    \includegraphics[width=0.75\textwidth]{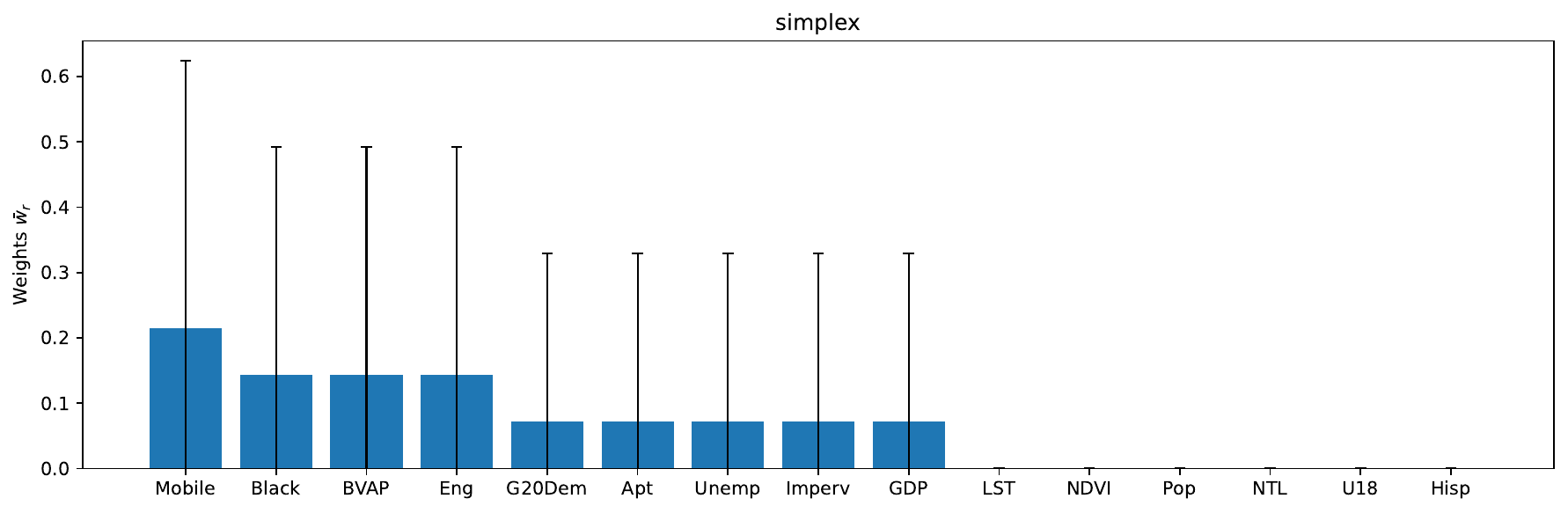}
    \caption{Mean suppression weights for fsFGW (Simplex) when comparing Plan 22 and Plan 22ct.}
    \label{fig:mean_weights_22_22ct}
\end{figure}

\textbf{Localized changes between districts} Finally, we examine a more localized redesign between Plan 23 (Senate Bill 757) and Plan 25 (Senate Bill 249), which differ only in two northeastern districts (c.f. Figure~\ref{fig:ncmaps}). Despite this limited geographic scope, fsFGW reveals systematic shifts in precinct composition aligned with targeted partisan refinement. Figure~\ref{fig:lasso_23_25} shows that nonzero Lasso weights are concentrated in only two district pairs (d0 and d7), while the remaining eleven pairs receive zero weight across all features, confirming that the redistricting change is highly localized. In d0, the active features reflect racial, ethnic, and linguistic minority composition (BVAP, Black, Hisp, Eng) and urbanization level (Mobile, LST). In d7, a broader set of features is active: racial and ethnic composition again appears through BVAP, Black, and Hisp, alongside political alignment (G20Dem), urbanization level (NDVI, LST), and housing characteristics including mobile homes and pre-1980 construction (Mobile, Pre80). Together, these patterns indicate that the revised districts are reconfigured along racial, political, and urbanization dimensions, consistent with targeted adjustments to electoral composition rather than broad structural change.

\begin{figure}[ht]
    \centering
    \includegraphics[width=\textwidth]{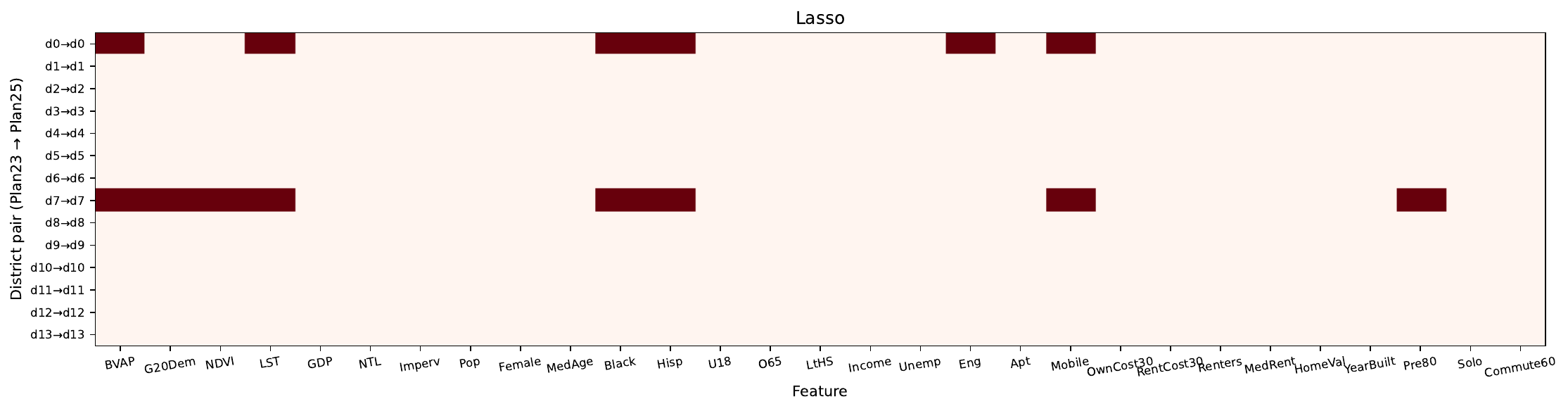}
    \caption{Suppression weights for fsFGW (Lasso) when comparing Plan 23 and Plan 25. Each row is a district pair and column is feature. See Figure~\ref{fig:ncweights23v25} for other modes.}
    \label{fig:lasso_23_25}
\end{figure}

Across these comparisons, fsFGW identifies feature-level differences that track the underlying redistricting process recently in North Carolina: court-imposed changes (22 to 22ct) produce distributed adjustments across districts consistent with bias reduction, while legislative revisions (23 to 25) induce highly localized, targeted shifts that refine partisan advantage.

\section{Discussion and Conclusions}
\label{sec:discussion}

The proposed fsFGW framework extends FGW by embedding feature selection directly into the transport objective, enabling identification of the features that drive dissimilarity under alignment instead of leaning on a uniform aggregation of features. It yields competitive distances relative to GW and FGW while providing interpretable outputs on why structured objects differ. Across synthetic and benchmark datasets, fsFGW emphasizes features with high alignment cost, and in redistricting reveals how geographically similar plans diverge along demographic and electoral dimensions.

\textbf{Limitations} At the same time, fsFGW inherits the non-convexity of FGW: the joint optimization over transport and feature weights may converge to local minima, with runtime dominated by repeated FGW solves, which can limit scalability. Performance also depends on the choice of regularization strength or suppression fraction, and different regularizers trade off sparsity and stability.

\textbf{Broader Impacts} fsFGW is relevant to applications requiring interpretable, feature-level attribution such as graph-structured data, biological networks, and settings such as redistricting, where identifying the features driving differences between plans can support transparency and analysis. However, suppression weights reflect features driving transport cost under a given alignment, not causal explanations, and should be interpreted accordingly.

Overall, fsFGW provides a flexible and principled framework for interpretable, structure-aware comparison in high-dimensional settings, embedding feature selection directly into optimal transport. Future work includes improving scalability and developing adaptive strategies for parameter selection.

\begin{ack}
This project was supported by the National Science Foundation (NSF) Grant No. DMS-2520375, while the authors attended the 2025 Research Collaboration Workshop in the Science of Data and Mathematics (WiSDM) held at University of North Carolina at Chapel Hill on August 4--8, 2025. R. A. Clark would also like to acknowledge the support of the NSF MSP Ascending Postdoc Award No. DMS-2138110. The authors thank Susan Glenn, Kathryn Leonard, Nkechi Nnadi, and Sarah Tymochko for early discussions, and Yifei Lou for organizing the WiSDM workshop.
\end{ack}

\newpage

\bibliography{references}
\bibliographystyle{plain}

\newpage
\appendix
\section{Additional Details for Theoretical Properties of fsFGW Distance}
\label{app:proofs}

With a mild abuse of notation, we will write $\mathrm{fsFGW}^\star(\mathcal{X}, \mathcal{Y})$ to denote the \textit{minimizing} fsFGW distance between objects $\mathcal{X}$ and $\mathcal{Y}$ as in \eqref{eq:fgw-unified} or \eqref{eq:fsfgw_unified_continuous}, and $\mathrm{fsFGW}(\mathbf{T}, \mathbf{w})$ to indicate the value of the fsFGW loss function evaluated at $(\mathbf{T}, \mathbf{w})$. $(\mathcal{X}, \mathcal{Y})$ may be omitted when the meaning is clear. In other words, $\mathrm{fsFGW}^\star = \mathrm{fsFGW}^\star(\mathcal{X}, \mathcal{Y}) = \inf_{T, w} \mathrm{fsFGW}(\mathbf{T}, \mathbf{w})$ given $\mathcal{X}, \mathcal{Y} = \mathrm{fsFGW}(\mathbf{T}^\star, \mathbf{w}^\star)$ given $\mathcal{X}, \mathcal{Y}$. Similar notations apply for GW and FGW.

\subsection{Continuous Formulation of fsFGW for Metric Measure Spaces}
Consistent with GW and FGW, the feature-suppressed FGW problem extends naturally to general metric spaces $(X, d_X)$ and $(Y, d_Y)$ with measures $\mu$ and $\nu$.
The continuous formulation of fsFGW is
\begin{equation}\label{eq:fsfgw_unified_continuous}
\min_{\pi \in \Pi(\mu,\nu),\, w \in \mathcal{W}}
\; (1-\alpha) \sum_{r=1}^d (1-w_r) s_r(\pi) + \alpha \mathrm{GW}(\pi)
+ \lambda R(w),
\end{equation}
where feature scores are defined as
\begin{equation}
    s_r(\pi) = \int_{\mathcal{X} \times \mathcal{Y}} 
\bigl|f_r(x) - g_r(y)\bigr|^q \, \mathrm{d}\pi(x,y)
\end{equation}
and the GW term is
\begin{equation}
    \mathrm{GW}(\pi) = \int_{(\mathcal{X} \times \mathcal{Y})^2}
\left|d_X(x,x') - d_Y(y,y')\right|^q 
\, \mathrm{d}\pi(x,y)\, \mathrm{d}\pi(x',y').
\end{equation}
The continuous and discrete formulations are connected via Corollary~\ref{cor:convergence}.

\subsection{Proof of Theorem \ref{thm:existence}: Existence of fsFGW}

\begin{proof}
Let $U(\mathbf{a},\mathbf{b})$ denote the transport polytope and let $\mathcal{W}$ denote the admissible set of suppression weights. We consider the following three cases:
\[
\mathcal{W} =
\begin{cases}
[0,1]^d & \text{(Lasso / Ridge)},\\
\Delta^{d-1} & \text{(Simplex)},\\
\Delta^{G-1} & \text{(Groupwise simplex)}.
\end{cases}
\]

First, $U(\mathbf{a},\mathbf{b})$ is a nonempty, closed, and bounded subset of $\mathbb{R}^{n \times m}$, hence compact. The set $\mathcal{W}$ is also compact in each case above, since it is either a closed hypercube ($[0,1]^d)$ or a probability simplex ($\Delta^{d-1}$ or $\Delta^{G-1}$). Therefore, the product set $U(\mathbf{a},\mathbf{b}) \times \mathcal{W}$ is compact.

Next, we show that the objective function in \eqref{eq:fgw-unified} is continuous on $U(\mathbf{a},\mathbf{b}) \times \mathcal{W}$. Note that fsFGW introduces a new bilinear feature term:
\[
(1-\alpha)\sum_{r=1}^d (1-w_r)s_r(\mathbf{T}).
\]
This is continuous, since $s_r(\mathbf{T}) = \sum_{i,j} T_{ij}|x_{ir}-y_{jr}|^q$ is linear in $\mathbf{T}$ and $(1-w_r)$ is linear in $\mathbf{w}$, so their product is continuous in $(\mathbf{T},\mathbf{w})$. The GW term $\mathrm{GW}(\mathbf{T})$ is a quadratic polynomial in $\mathbf{T}$ and hence continuous. The regularization term $\lambda R(\mathbf{w})$ is continuous for both Lasso and Ridge, and vanishes in the simplex-constrained cases.

Thus, the objective function is continuous on a compact set. By the extreme value theorem, it attains its minimum on $U(\mathbf{a},\mathbf{b}) \times \mathcal{W}$. It follows that there exists a minimizer $(\mathbf{T}^\star, \mathbf{w}^\star)$.
\end{proof}

\subsection{Proof of Theorem \ref{thm:bounds}: Bounds on fsFGW distance}

\textbf{Lower bound.}
Let $(\textbf{T}^\star, \textbf{w}^\star)$ be an optimal solution of the fsFGW problem~\eqref{eq:fgw-unified}. For all four modes, 
\begin{equation}
       \mathrm{fsFGW}^\star = 
(1-\alpha)\sum_{r=1}^d (1-w^\star_r)\,s_r(\mathbf{T}^\star)
+
\alpha\,\mathrm{GW}(\mathbf{T}^\star)
+
\lambda R(\mathbf{w}^\star)
\ge \alpha\,\mathrm{GW}(\mathbf{T}^\star) \ge \alpha \mathrm{GW}^\star.
\end{equation}

\textbf{Upper bound.}
\emph{Lasso and Ridge} follows by evaluating at $(\mathbf{T}^\star_{FGW}, \mathbf{0})$ being a feasible point.
\begin{equation}
       \mathrm{fsFGW}^\star  \le \mathrm{fsFGW}(\mathbf{T}^\star_{FGW}, \mathbf{0}) = \mathrm{FGW}^\star.
\end{equation}

For \emph{Simplex and Groupwise Simplex}, we instead consider $\mathbf{w}^{\dagger} = e_{r^\star}$ where  $r^\star \in \arg\min_r s_r(\mathbf{T}^\star_{\mathrm{FGW}})$ is the least differentiating feature under the classical FGW plan. Then $\mathbf{w}^{\dagger} \in \Delta^{d-1}$ and $\lambda = 0$,  giving
\begin{equation}
     \mathrm{fsFGW}^\star \le  \mathrm{fsFGW}(\mathbf{T}^\star_{FGW}, \mathbf{w}^\dagger) = \mathrm{FGW}^\star - (1-\alpha)\,s_{r^\star}(\mathbf{T}^\star_{\mathrm{FGW}}) \le \mathrm{FGW}^\star.
\end{equation}

\subsection{Convergence of Finite Samples} 
We start by discussing what \cite{vayer2020fused} did for FGW. 
From the continuous formulation of FGW on metric space $(X, d_X)$ with measure $\mu$, consider the following. We can sample from the joint distribution to obtain a sequence of empirical measures for $n\in \mathbb{N}$
\begin{equation}
   \mathcal{X}_n=  (\{x_i\}_{i=1}^n,\, d_{X},\, \mu_n ), \label{eq:sequences}
\end{equation}
where
\begin{equation}
\mu_n = \frac{1}{n} \sum_{i=1}^n \delta_{a_i},
\qquad (x_i,a_i) \sim \mu.
\end{equation}
The authors asked whether this sequence converges to $(X, d_X,\mu)$ in the FGW sense for $q=1$, i.e.
\begin{equation}
    \lim_{n\to \infty} \mathrm{FGW}^\star(\mathcal{X}_n, \mathcal{X})=0
\end{equation}
and at what rate this convergence occurs. \cite[Theorem 2]{vayer2020fused} uses $\dim_p^*(\mu)$, its \emph{upper Wasserstein dimension}, to characterize the convergence rate. They state that when the measure is sufficiently regular, this quantity coincides with the intuitive notion of dimension. We refer the reader to \cite{vayer2020fused} and references within for more details.

The bounds in Thoerem \ref{thm:bounds} give a straightforward extension of \cite[Theorem 2]{vayer2020fused} to our fsFGW distance.

\begin{corollary}[Convergence of finite samples]
\label{cor:convergence}
For $q=1$, the sequence of structured objects in \eqref{eq:sequences} converges in the $\mathrm{fsFGW}$ sense, meaning
\begin{equation}
    \lim_{n\to\infty} \mathrm{fsFGW}^\star(\mathcal{X}_n, \mathcal{X}) = 0.
\end{equation} 
Moreover, 
\begin{equation}
    \mathbb{E}[\mathrm{fsFGW}^\star(\mathcal{X}_n, \mathcal{X})] \le Cn^{-1/s}
\end{equation}
for $s>\dim_p^*(\mu)$.
\end{corollary}
Proof directly follows from a sandwiching argument. Because the sequence of finite samples converges in both the FGW sense and in the GW sense, fsFGW does too.

\subsection{Proof of Theorem \ref{thm:metric}: Metric Properties}

Showing \textbf{Positivity} and \textbf{Symmetry} is trivial.

\textbf{Identity of indiscernibles.}

For \textit{Lasso} and \textit{Ridge}, if $\mathrm{fsFGW}^\star(\mathcal{X}, \mathcal{Y}) = 0$, due to non-negativity of every term in \eqref{eq:fgw-unified}, it must be that $\lambda R(\mathbf{w}^\star)=0$. For Lasso and Ridge, $\|\mathbf{w}^\star\|_1$ and $\|\mathbf{w}^\star\|_2^2$ are 0 only when $\mathbf{w}^\star = \mathbf{0}$. So $\mathrm{fsFGW}^\star$ reduces to the classical FGW, which is shown to satisfy the identity of indiscernibles in this direction in~\cite[Theorem 1]{vayer2020fused}.

In the other direction, we wish to show that $\mathcal{X}=\mathcal{Y} \implies \mathrm{fsFGW}^\star(\mathcal{X}, \mathcal{Y})=0$. When $\mathcal{X}=\mathcal{Y}$, $\mathrm{fsFGW}(\mathbf{I}, \mathbf{0})=0$ can be achieved. Then due to non-negativity of individual terms in \eqref{eq:fgw-unified}, $0 \le \mathrm{fsFGW}^\star \le \mathrm{fsFGW}(\mathbf{I}, \mathbf{0}) = 0 \implies \mathrm{fsFGW}^\star =0$.

For \textit{Simplex} and \textit{Groupwise Simplex}, one can construe a counter-example. Two graphs that differ on the suppressed feature(s) but agree structurally and on all other features satisfy $\mathrm{fsFGW}^\star = 0$ without being isomorphic.

\textbf{Triangle inequality.}

We prove the triangle inequality for fsFGW with Lasso and Ridge regularization. 

\textit{Step 1: Reduce to optimization over $\mathbf{T}$ only.}

Note that for fixed $\mathbf{T}$, the objective is separable in $\mathbf{w}$, so we can minimize over $\mathbf{w}$ pointwise, yielding an equivalent problem over $\mathbf{T}$ only. Also, observe the optimal \underline{Lasso} weights are given in closed form by $w_r^\star = \mathbf{1}[(1-\alpha)s_r(\mathbf{T}) > \lambda]$. Substituting into the full objective:

\begin{align*}
&(1-\alpha)\sum_r (1-w_r^\star)s_r(\mathbf{T}) + \alpha\,\mathrm{GW}(\mathbf{T}) + \lambda\sum_r w_r^\star \\
&= \sum_r \left[(1-\alpha)(1-w_r^\star)s_r(\mathbf{T}) + \lambda w_r^\star\right] + \alpha\,\mathrm{GW}(\mathbf{T})\\
&= \sum_r \begin{cases} (1-\alpha)s_r(\mathbf{T}) & \text{if } (1-\alpha)s_r(\mathbf{T}) \leq \lambda \\ \lambda & \text{if } (1-\alpha)s_r(\mathbf{T}) > \lambda \end{cases} + \alpha\,\mathrm{GW}(\mathbf{T})\\
&= \sum_r \min\left((1-\alpha)s_r(\mathbf{T}), \lambda\right) + \alpha\,\mathrm{GW}(\mathbf{T}).
\end{align*}

For fixed $\mathbf{T}$, the optimal \underline{Ridge} weights are $w_r^\star = \min\left(1, \frac{(1-\alpha)s_r(\mathbf{T})}{\lambda}\right)$. Substitute into the full objective and consider the terms related to $r$:

Case 1: $(1-\alpha)s_r(\mathbf{T}) \leq \lambda$, so $w_r^\star = \frac{(1-\alpha)s_r(\mathbf{T})}{\lambda}$:
\begin{align*}
(1-\alpha)(1-w_r^\star)s_r(\mathbf{T}) + \frac{\lambda}{2}w_r^{\star2} &= (1-\alpha)\left(1 - \frac{(1-\alpha)s_r(\mathbf{T})}{\lambda}\right)s_r(\mathbf{T}) + \frac{\lambda}{2}\left(\frac{(1-\alpha)s_r(\mathbf{T})}{\lambda}\right)^2\\
&= (1-\alpha)s_r(\mathbf{T}) - \frac{(1-\alpha)^2 s_r(\mathbf{T})^2}{2\lambda}
\end{align*}

Case 2: $(1-\alpha)s_r(\mathbf{T}) > \lambda$, so $w_r^\star = 1$:
$$(1-\alpha)(1-1)s_r(\mathbf{T}) + \frac{\lambda}{2}\cdot 1 = \frac{\lambda}{2}$$

So summarizing both types of regularization, we have:
\begin{equation}
   \mathrm{fsFGW}^\star(\mathcal{X},\mathcal{Z}) = \min_{\mathbf{T} \in U(\mathbf{a},\mathbf{c})} \left[\sum_r g(s_r(\mathbf{T})) + \alpha\,\mathrm{GW}(\mathbf{T})\right], \label{eq:only_T} 
\end{equation}
where 
\begin{equation}
    g(s) = \min\left((1-\alpha)s, \lambda\right) \label{eq:g_lasso}
\end{equation}
for Lasso, and 
\begin{equation}
    g(s) = \begin{cases} (
    1-\alpha)s - \frac{(1-\alpha)^2 s^2}{2\lambda} & \text{if } (1-\alpha)s \leq \lambda \\ 
    \frac{\lambda}{2} & \text{if } (1-\alpha)s > \lambda \end{cases} \label{eq:g_ridge}
\end{equation}
for Ridge.

\textit{Step 2: Construct a feasible transport plan via gluing.}

Let $\mathbf{T}_1^\star \in U(\mathbf{a},\mathbf{b})$ and $\mathbf{T}_2^\star \in U(\mathbf{b},\mathbf{c})$ be optimal transport plans for $\mathrm{fsFGW}^\star(\mathcal{X},\mathcal{Y})$ and $\mathrm{fsFGW}^\star(\mathcal{Y},\mathcal{Z})$ respectively after the reduction in Step 1. By the Gluing Lemma, there exists a coupling $\pi$ on $X \times Y \times Z$ with marginals $\mathbf{T}_1^\star$ on the first two coordinates and $\mathbf{T}_2^\star$ on the last two coordinates. Let $\mathbf{T}_3$ denote the marginal of $\pi$ on $X \times Z$. By construction $\mathbf{T}_3 \in U(\mathbf{a},\mathbf{c})$.

\textit{Step 3: Bound $s_r(\mathbf{T}_3)$ and $\mathrm{GW}(\mathbf{T}_3)$.}
By the same argument as Proposition 4 in \cite{vayer2020fused},
\begin{align}
s_r(\mathbf{T}_3) &\leq C_q\left(s_r(\mathbf{T}_1^\star) + s_r(\mathbf{T}_2^\star)\right),\\
\mathrm{GW}(\mathbf{T}_3) &\leq C_q\left(\mathrm{GW}(\mathbf{T}_1^\star) + \mathrm{GW}(\mathbf{T}_2^\star)\right)
\end{align}
for $C_q=1$ when $q=1$ and $C_q=2^{q-1}$ for $q\ge 2$.

\textit{Step 4: Bound the feature term.}

We wish to show that for $C_q\ge 1, \lambda>0, s\ge 0, 1-\alpha \ge 0$, the $g$ defined in \eqref{eq:g_lasso} and \eqref{eq:g_ridge} satisfy
\begin{align}
    g(s) &\le g(C_q(s_1 + s_2))\\
    &\le C_q ~g(s_1+ s_2)\\
    &\le C_q ~g(s_1) + C_q ~ g(s_2).
\end{align}
Because $g$ is monotonically non-decreasing, it produces the first inequality when applied to the bound from Step 3. Second and third inequalities follow from checking that $g$ is concave with $g(0) = 0$. For both Lasso and Ridge, the function $g$ is concave, nondecreasing, and satisfies $g(0)=0$. For Lasso, $g(s)=\min((1-\alpha)s,\lambda)$ is the minimum of affine functions; for Ridge, $g$ is piecewise concave with a quadratic segment followed by a constant region. More precisely,
\begin{itemize}
    \item  The second inequality follows from Jensen's inequality applied to the concave function $g$: since $C_q \geq 1$, we have $\frac{1}{C_q} \in [0,1]$ and $s = \frac{1}{C_q}(C_q s) + (1-\frac{1}{C_q})(0)$, so by concavity:
$$g(s) \geq \frac{1}{C_q}g(C_q s) + \left(1-\frac{1}{C_q}\right)g(0) = \frac{1}{C_q}g(C_q s)$$ 
rearranging gives $g(C_q s) \leq C_q g(s)$.

\item Since $g$ is concave with $g(0)=0$, it is subadditive, i.e.,
\[
g(s_1+s_2) \le g(s_1) + g(s_2).
\] For any $s_2 \ge 0$, the third inequality follows from:
$$g(s_1+s_2) - g(s_2) \leq g(s_1+0) - g(0) = g(s_1)$$
rearranging gives subadditivity.
\end{itemize}

Both properties hold for Lasso and Ridge since in each case $g$ is concave, nondecreasing, and satisfies $g(0)=0$.

\textit{Step 5: Combine the bounds.}

By feasibility of $\mathbf{T}_3$ for $\mathrm{fsFGW}^\star(\mathcal{X},\mathcal{Z})$, and using Steps 3 and 4:

\begin{align*}
\mathrm{fsFGW}^\star(\mathcal{X},\mathcal{Z}) &\leq \sum_r g(s_r(\mathbf{T}_3))+ \alpha\,\mathrm{GW}(\mathbf{T}_3)\\
&\leq C_q\sum_r\left(g(s_r(\mathbf{T}^\star_1))+g(s_r(\mathbf{T}^\star_2))\right) + \alpha C_q\left(\mathrm{GW}(\mathbf{T}_1^\star) + \mathrm{GW}(\mathbf{T}_2^\star)\right)\\
&= C_q\left(\mathrm{fsFGW}^\star(\mathcal{X},\mathcal{Y}) + \mathrm{fsFGW}^\star(\mathcal{Y},\mathcal{Z})\right).
\end{align*}
This gives the (relaxed) triangle inequality as claimed. This highlights the importance of Step 1: the triangle inequality holds because the joint optimization over $(\mathbf{T},\mathbf{w})$ reduces to a separable concave transformation of feature costs.

\section{Algorithm and Convergence}
\label{app:convergence}

\begin{algorithm}[htp]
\caption{Feature-Selected Fused Gromov--Wasserstein (fsFGW)}
\label{alg:sfgw}
\begin{algorithmic}[1]
\Require
Structure cost matrices $\mathbf{C}_X, \mathbf{C}_Y$,
feature cost matrices $\mathbf{M}_r = [|x_{ir}-y_{jr}|^q]_{i,j}$ for $r\in[1,d]$,
distributions $\mathbf{a}, \mathbf{b}$,
trade-off $\alpha \in (0,1)$,
suppression mode (Lasso, Ridge, Simplex or Group Simplex),
suppression parameter $f$ or $\lambda$. For Group Simplex, group assignments $\{\mathcal{G}_1, \ldots, \mathcal{G}_G\}$.

\State Initialize $\mathbf{w}^{(0)} = \mathbf{0}$, $k \gets 0$
\State Compute initial transport $\mathbf{T}^{(0)}$ by solving fsFGW with $\mathbf{w}^{(0)}$, which is equivalent to classical FGW. 

\If{$\lambda$ not provided}
    \State Compute feature scores $s_r(\mathbf{T}^{(0)})$ for $r=1,\dots,d$
    \State Set $\lambda = (1-\alpha)\, Q_f\!\left(s_1(\mathbf{T}^{(0)}),\dots,s_d(\mathbf{T}^{(0)})\right)$
\EndIf

\Repeat

    \State \textbf{Weight update:}
    \State Compute feature scores $s_r(\mathbf{T}^{(k)})$

    \If{Lasso}
        \State $w_r^{(k+1)} \gets \mathbb{I}\!\left[(1-\alpha)s_r(\mathbf{T}^{(k)}) > \lambda \right]$
    \ElsIf{Ridge}
        \State $w_r^{(k+1)} \gets \min\!\left\{1,\; \frac{(1-\alpha)s_r(\mathbf{T}^{(k)})}{\lambda}\right\}$
    \ElsIf{Simplex}
        \State $w^{(k+1)} \gets \mathbf{e}_{r^\star}$, where $r^\star = \arg\max_r s_r(\mathbf{T}^{(k)})$
    \ElsIf{Group Simplex}
        \State $g^\star \gets \arg\max_g \sum_{r \in \mathcal{G}_g} s_r(\mathbf{T}^{(k)})$
        \State $w^{(k+1)} \gets \mathbf{e}_{g^\star}$
    \EndIf

    \State \textbf{Transport update:}
    \State Compute $\mathbf{T}^{(k+1)}$ by solving classical FGW with weighted feature cost
    \[
    \mathbf{M}^{(k)} = \sum_{r=1}^d (1-w_r^{(k)})\,\mathbf{M}_{r}.
    \]

    \State $k \gets k+1$
\Until{convergence}

\Return $\mathbf{T}^{(k)}, \mathbf{w}^{(k)}$
\end{algorithmic}
\end{algorithm}

\begin{figure}
    \centering
    \includegraphics[width=\linewidth]{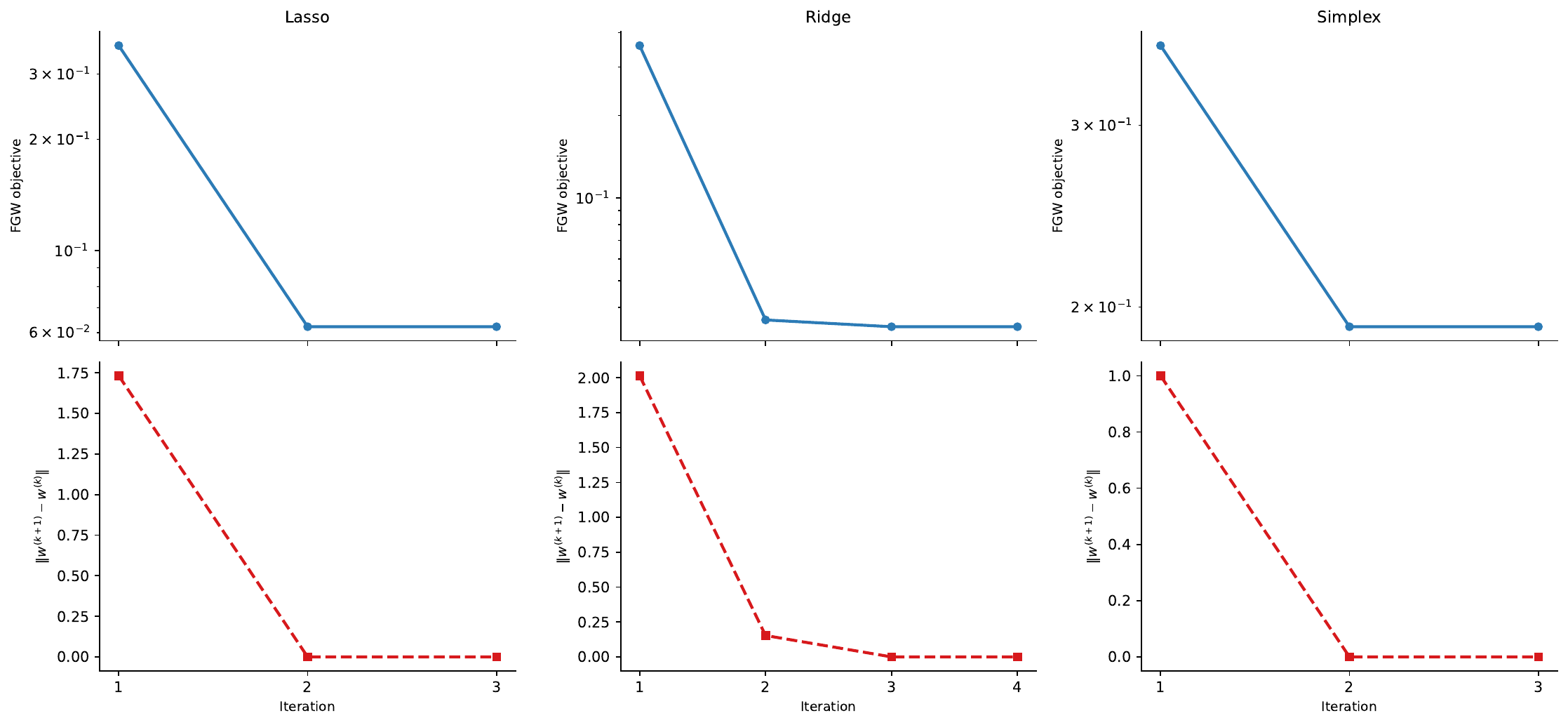}
    \caption{Convergence of fsFGW with $f=0.3$ on two graphs (Graphs 20 and 26) from \textsc{ogbg-molbace}. Top: FGW objective per iteration. Bottom: weight change $\|\mathbf{w}^{k+1} - \mathbf{w}^k\|$. All modes converged in 2-8 iterations in our observations.}
    \label{fig:convergence}
\end{figure}

Our fsFGW algorithm is summarized in Algorithm~\ref{alg:sfgw}. Figure~\ref{fig:convergence} shows the FGW objective and weight change $\|\mathbf{w}^{(k+1)} - \mathbf{w}^{(k)}\|$ per iteration on two graphs (Graphs 20 and 26) from \textsc{ogbg-molbace}. All three modes converge in fewer than five iterations, with Lasso and Simplex converging fastest due to their binary updates. The objective drops steeply from iteration~0 to~1 --- the first $\mathbf{w}$-update after the classical FGW initialization produces a large improvement --- and then flattens.  Most runs that we observed converged in 2-8 steps. The per-iteration cost is dominated by the inner FGW solve; the closed-form $\mathbf{w}$-update adds negligible overhead regardless of $d$. 

All experiments in this paper, including application in computational redistricting, was run on free-tier Google Colab using Intel(R) Xeon(R) CPU @ 2.20GHz with 12GB memory. All GW and FGW pairwise distances are computed using the Python Optimal Transport library~\cite{flamary2021pot,flamary2024pot}.

\section{Additional Experiment Details: Synthetic Data}
\label{app:synthetic}
The low-dimensional instance uses $d=10$ features, $k=3$ differentiating, $d-k=7$ shared, and $n=40$ nodes. The high-dimensional instance uses $d=100$ features, $k=10$ differentiating, $d-k=90$ shared, and $n=40$ nodes. Feature cost matrices $\mathbf{M}_{r}$ are normalized per-feature to $[0,1]$. Lambda is calibrated at each instance using the initial FGW plan $\mathbf{T}_0$.

\textbf{Low-dim: Identification of Differentiating Features}
\begin{figure}[htp]
    \centering
    \includegraphics[width=\linewidth]{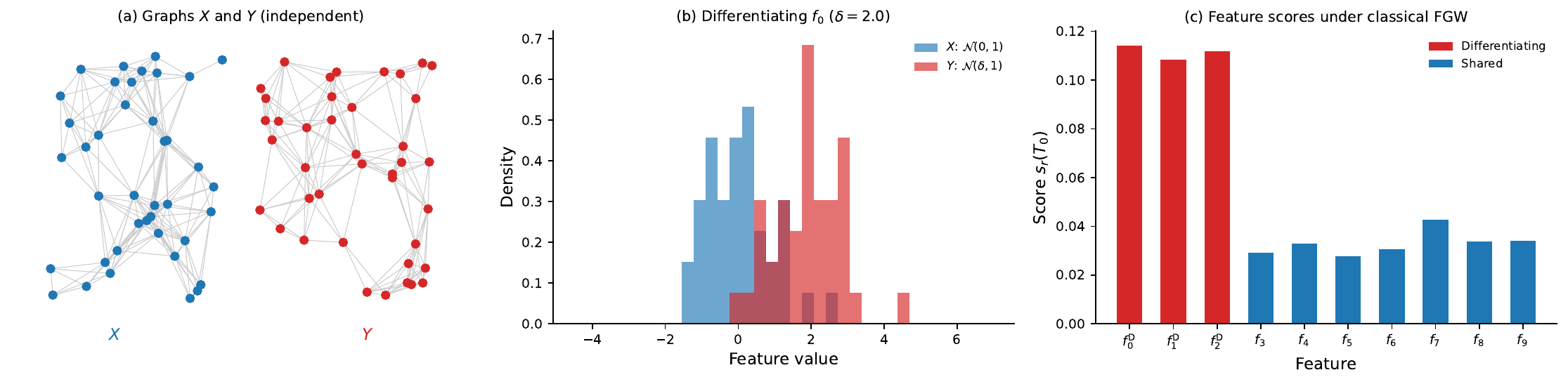}
    \caption{(a) The two independent random geometric graphs $X$ (blue) and $Y$ (red). (b) A differentiating feature: the distributions in $X$ and $Y$ are separated. (c) Feature scores $s_r(\mathbf{T}_0)$ under classical FGW: differentiating features (red) score $3$--$4\times$ higher than shared features (blue).}
    \label{fig:exp1_setup}
\end{figure}

\begin{figure}[htp]
    \centering
    \includegraphics[width=\linewidth]{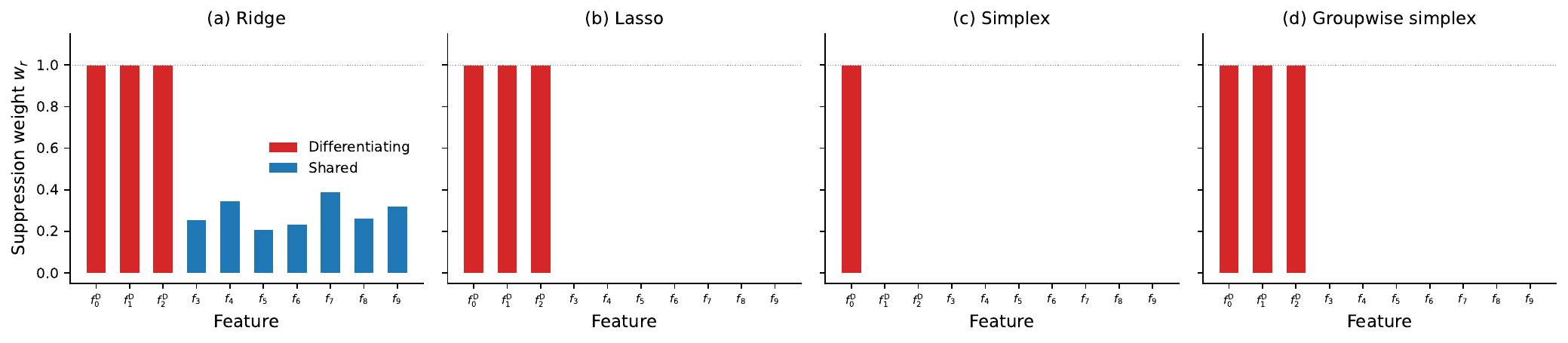}
    \caption{Suppression weights $w_r$ recovered by each mode ($n=40$, $d=10$, $k=3$, $\delta=2.0$). Differentiating features (red) should have $w_r=1$; shared (blue) should have $w_r=0$. Ridge ($f=0.2$) gives continuous weights; Lasso ($f=0.3$) gives binary weights recovering exactly the $k=3$ differentiating features. Simplex concentrates all  suppression on the single highest-scoring feature. Groupwise simplex suppresses the entire group $\mathcal{G}_0$ containing the differentiating features.}
    \label{fig:exp2_baseline}
\end{figure}

Figure~\ref{fig:exp1_setup} shows the two graphs and the feature score distribution under the classical FGW plan $\mathbf{T}_0$, obtained by solving~\eqref{eq:fgw} with $\mathbf{w}=\mathbf{0}$. Differentiating features score $3$--$4\times$ higher than shared features, confirming the signal is visible before any suppression is applied.

Figure~\ref{fig:exp2_baseline} shows the recovered suppression weights $w_r$ for all four modes on a low-dimensional instance ($d=10$, $k=3$, $\delta=2.0$). \textbf{Lasso} ($f=0.3$) recovers the exact differentiating set $\{f_0,f_1,f_2\}$ with binary weights and no false positives or negatives. \textbf{Ridge} ($f=0.2$) gives continuous weights: differentiating features saturate near $w_r=1$ while shared features receive partial suppression proportional to their scores. \textbf{Simplex} concentrates all suppression mass on the single highest-scoring feature, correctly identifying the strongest differentiating feature but not the full set. \textbf{Groupwise simplex}, given the correct partition $\mathcal{G}_0=\{f_0,f_1,f_2\}$, suppresses the entire differentiating group and assigns zero weight to all shared groups.

\begin{figure}
    \centering
    \includegraphics[width=\linewidth]{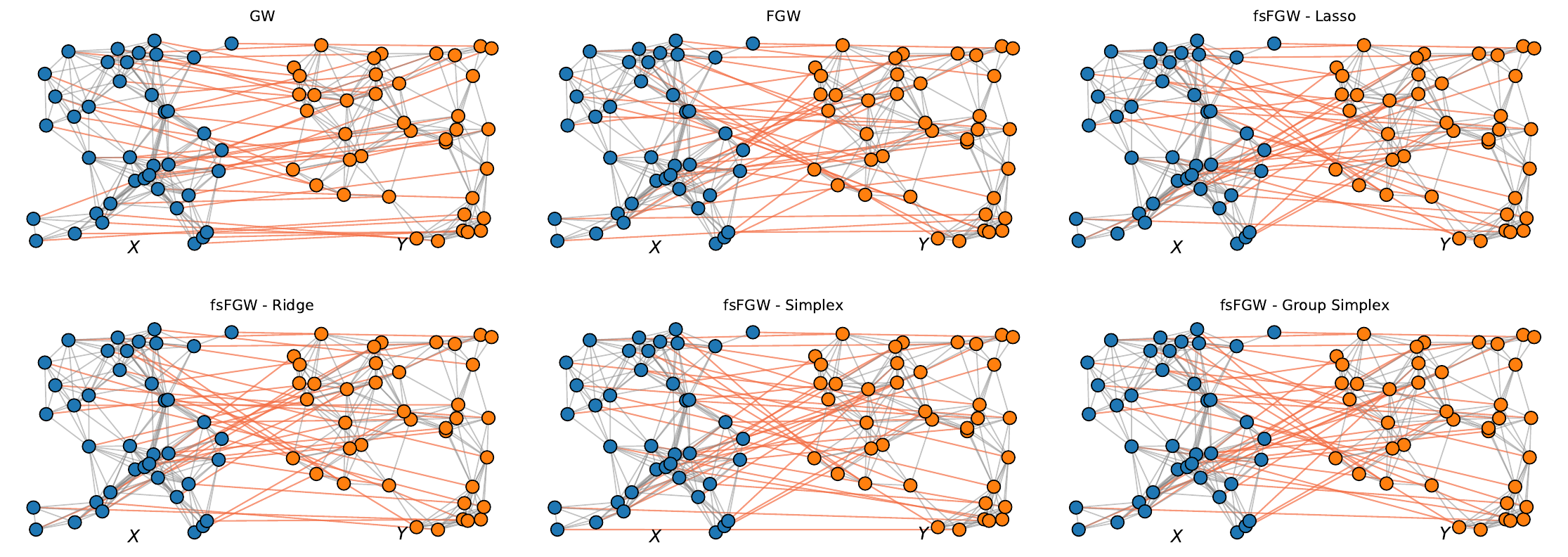}
    \caption{Visualization of matchings according to different transport plans $\mathbf{T}$. The four modes of fsFGW generate substantially distinct transport plans from each other, as well as from GW and FGW.}
    \label{fig:matched_graphs}
\end{figure}

\begin{figure}
    \centering
    \includegraphics[width=0.6\linewidth]{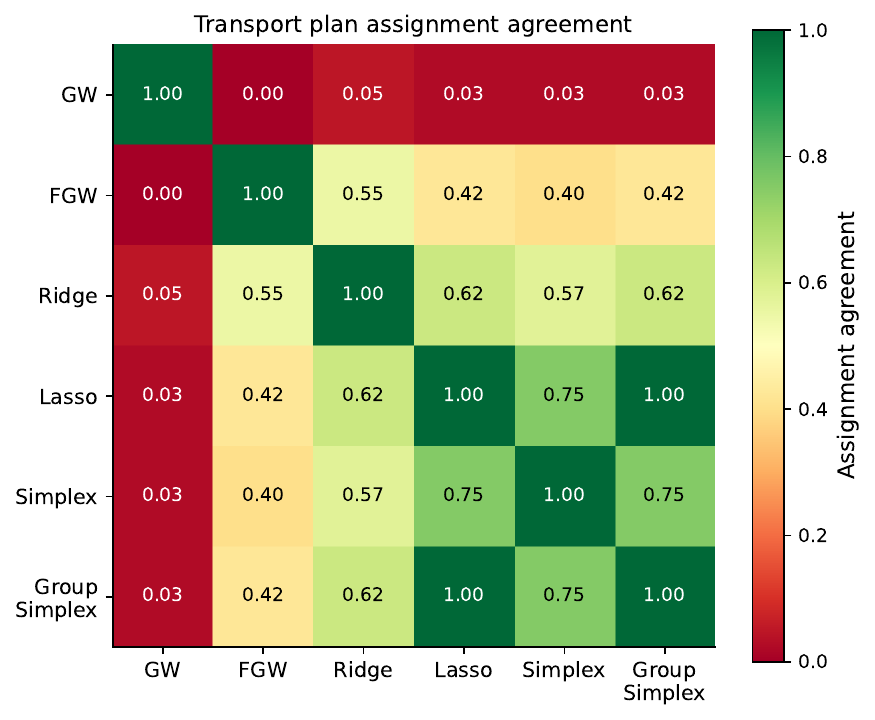}
    \caption{Agreement between different $\mathbf{T}$s, calculated as proportion of node mappings that agree between two transport plans.}
    \label{fig:agreement_heatmap}
\end{figure}

\section{Additional Experiment Details: Graph Benchmarks}
\label{app:benchmark}

\textsc{Frankenstein}~\cite{orsini2015graph} and \textsc{Proteins-full}~\cite{borgwardt2005protein,dobson2003distinguishing} are distributed through 
TUDataset~\cite{morris2020tudataset} and PyTorch-Geometric~\cite{Fey/Lenssen/2019,Fey/etal/2025} under MIT License.
\textsc{ogbg-molbace}~\cite{hiv,wu2018moleculenet} is released under the MIT License through the 
Open Graph Benchmark~\cite{hu2020open,hu2021ogblsc}. We use them solely for non-commercial research purposes consistent with 
their established use in the graph learning literature. 
\begin{figure}
    \centering
    \includegraphics[width=\linewidth]{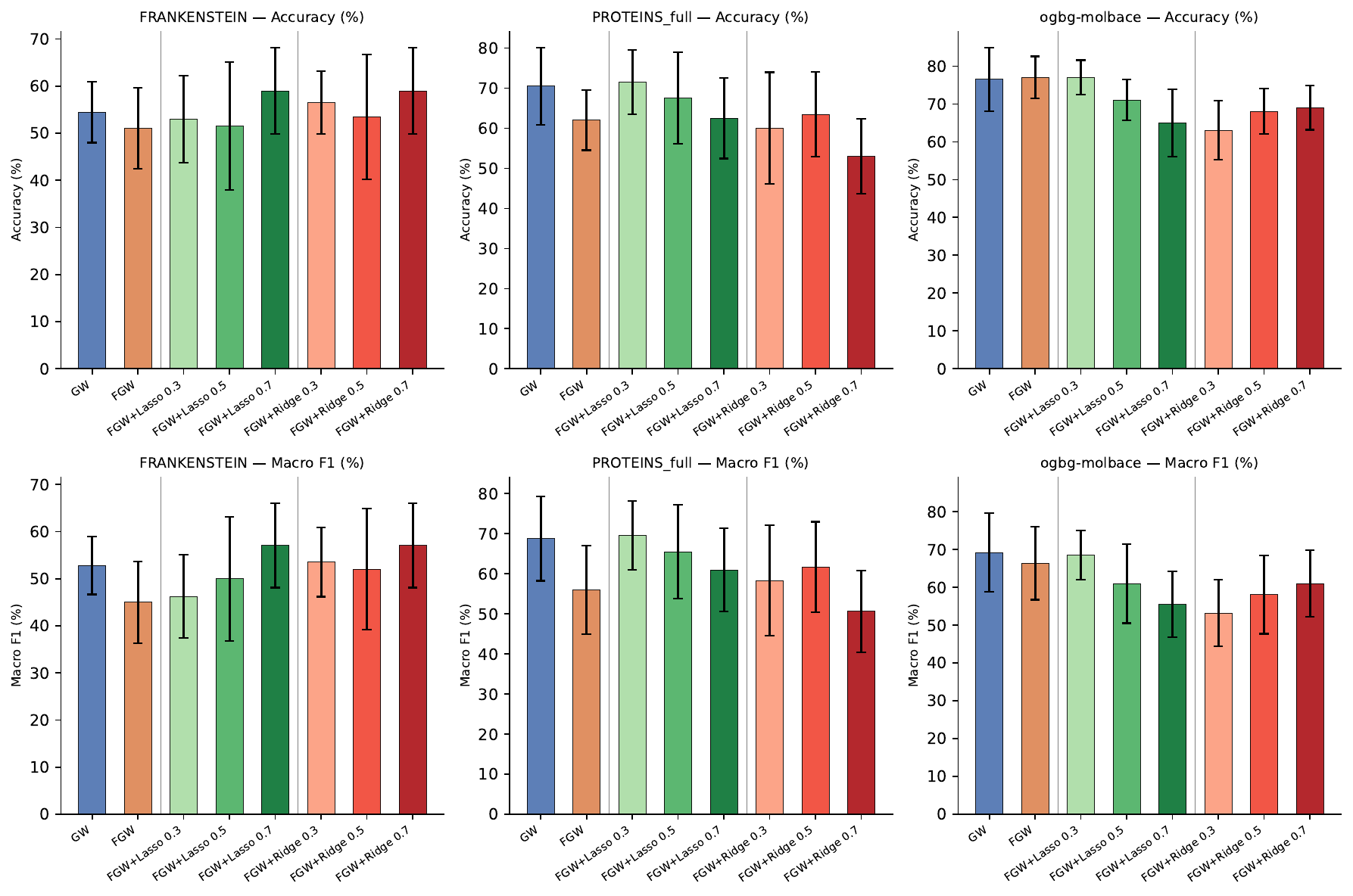}
    \caption{Classification results in Table~\ref{tab:benchmark}.}
    \label{fig:classification}
\end{figure}

\begin{figure}
    \centering
    \includegraphics[width=\linewidth]{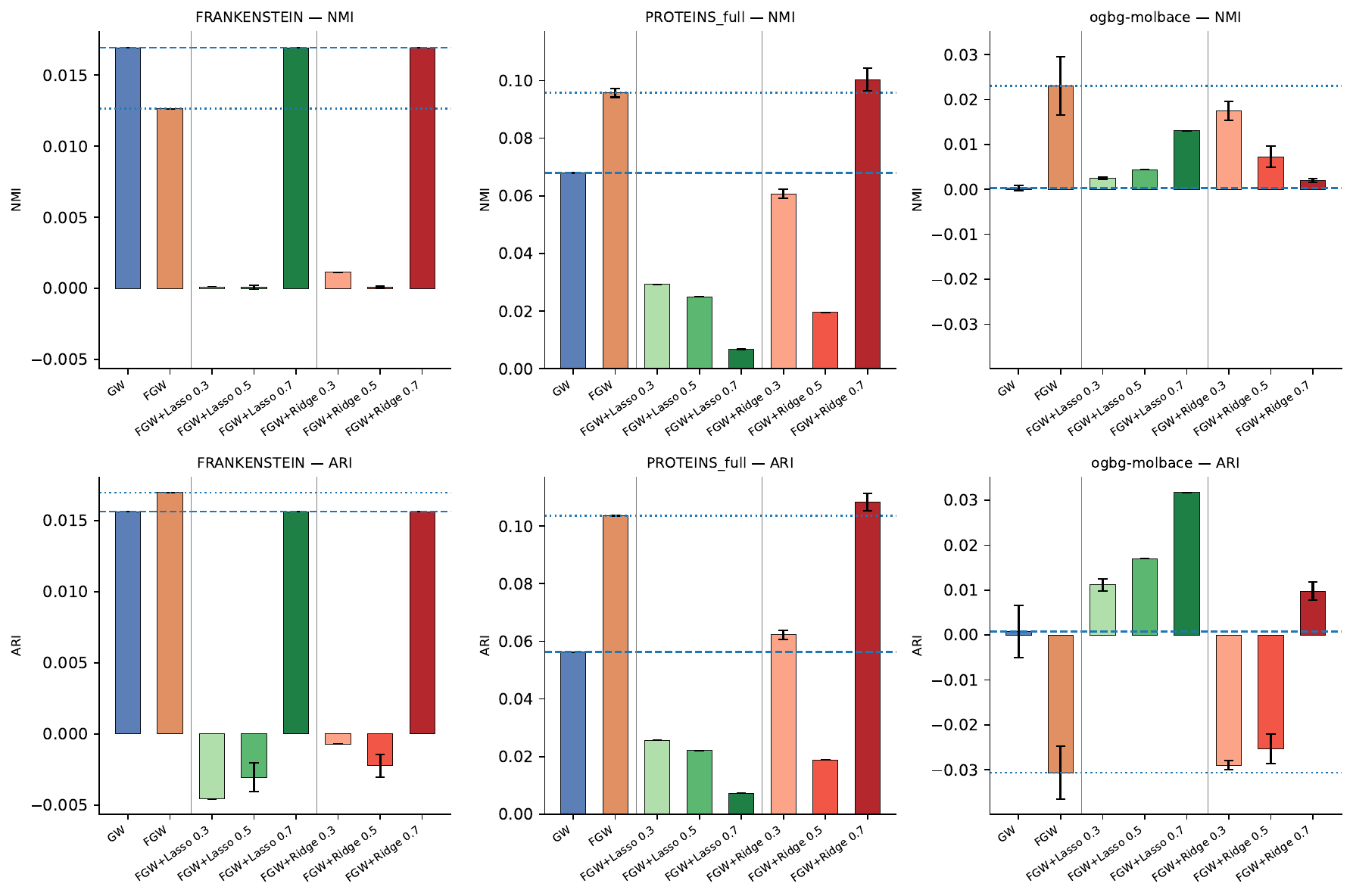}
    \caption{Clustering results in Table~\ref{tab:benchmark}.}
    \label{fig:clustering}
\end{figure}

Graphs with more than 60 (30 for \textsc{ogbg-molbace}) nodes are excluded to keep computation fast. From the remaining graphs, a stratified subsample of 200 graphs is drawn. Node features are normalized globally to $[0,1]$ per dimension across the subsample. Per-feature cost matrices $\mathbf{M}_{r}$ are further normalized per pair to $[0,1]$.

Classification uses an SVM with RBF kernel $K_{ij} = \exp(-D_{ij}/\sigma)$ where $\sigma$ is the median of nonzero distances, evaluated via 5-fold stratified CV. Clustering uses $k$-means with $k$ equal to the number of classes, averaged over 10 random seeds.

Figures~\ref{fig:classification} and \ref{fig:clustering} visualize the results from Table~\ref{tab:benchmark}. Lastly, Figure~\ref{fig:weights_proteins} shows the weights for \textsc{Proteins-full}. There is a clear group behavior on f21-f28, and f31 is almost always retained. However, we were not able to find what each of the features correspond to despite our best attempts at literature search.
\begin{figure}[ht]
    \centering
    \includegraphics[width=\linewidth]{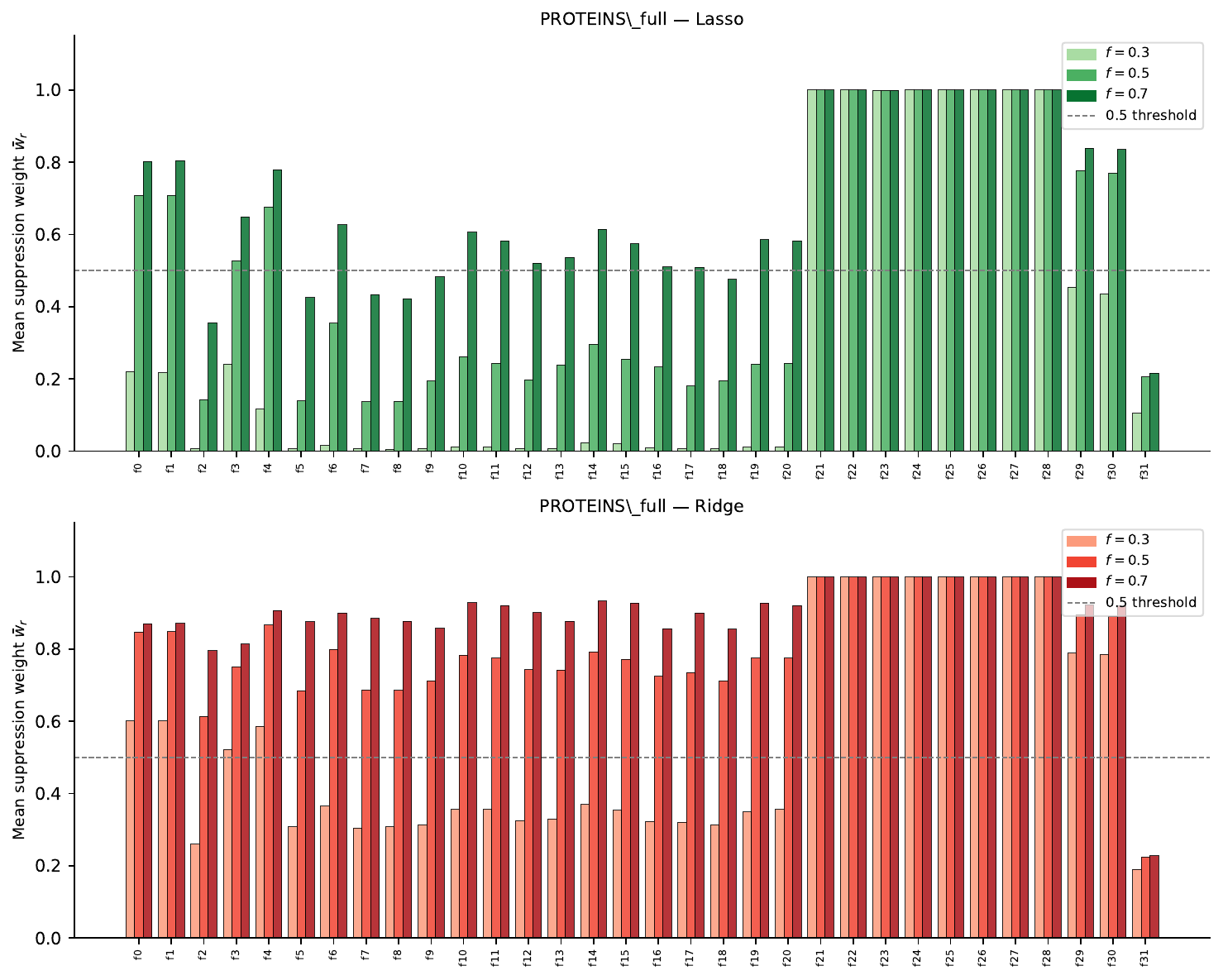}
    \caption{Mean suppression weights per feature on \textsc{Proteins-full}    ($d=32$) for Lasso (left) and Ridge (right) across suppression fractions  $f \in \{0.3, 0.5, 0.7\}$. Errorbars are omitted for readability.}
    \label{fig:weights_proteins}
\end{figure}

\section{Additional Details for Application in Computational Redistricting}
\label{app:redistricting}

\subsection{Experimental Setup}
We apply FGW with feature suppression to compare congressional redistricting plans for North Carolina. Six plans enacted between 2020 and 2025 are considered: House Bill 1029 (2020), Senate Bill 740 (2021), Senate Bill 745 (2022), the court-ordered plan from \textit{Harper v.\ Hall} (2022), Senate Bill 757 (2023), and Senate Bill 249 (2025). Each plan partitions NC's 2,650 voting precincts into 14 congressional districts. We represent each plan as a collection of districts, where each district is a subgraph of the precinct adjacency graph induced by its constituent precincts. 29 node features are the precinct-level sociodemographic and environmental indicators described in the next subsection. The structure matrix $\mathbf{C}$ for each district is the normalized geodesic distance matrix on its precinct subgraph.

To compare two plans $\mathcal{P}$ and $\mathcal{Q}$, we first establish a one-to-one correspondence between their 14 districts using the Hamming distance on precinct membership vectors: district $i$ in $\mathcal{P}$ is represented as a binary vector over all precincts, and the optimal one-to-one matching between the 14 districts of $\mathcal{P}$ and $\mathcal{Q}$ is obtained by solving a linear assignment problem on the $14 \times 14$ Hamming distance matrix. This matching identifies corresponding districts by geographic overlap, independent of district numbering. For each matched district pair $(i, j)$, we compute the GW, FGW, or fsFGW distance between the two district subgraphs. 
The total plan distance is the sum of FGW distances across all 14 matched district pairs, and the reported suppression weights are averaged over the 14 district pairs. Pairwise distances among all $\binom{6}{2} = 15$ plan pairs are used as input to hierarchical clustering with complete linkage. $\alpha=0.5$ and $f=0.1$ for fsFGW Ridge and $f=0.2$ for Lasso.

\subsection{NC Data Feature Processing}\label{app:features}

Precinct boundary geometries, demographic information, and election results were obtained from the Quantifying Gerrymandering group at Duke University (\url{https://quantifyinggerrymandering.pages.oit.duke.edu/codedoc/}). These data products are constructed by merging U.S.\ Census demographic data with precinct-level returns from the North Carolina State Board of Elections. Geographic shapefiles and associated attributes (including BVAP and 2020 election outcomes) are provided through the project’s public data repository and documentation.

\subsection{NC Data Feature Processing}\label{app:features}

We construct a precinct-level dataset for North Carolina by integrating sociodemographic, environmental, and built-environment indicators from census and remotely sensed data sources. Precinct boundaries serve as the spatial unit of redistricting analysis.

The sociodemographic variables are derived from the 2020 American Community Survey (ACS) 5-year estimates~\cite{acs2020}, published by the U.S. Census Bureau and available in the public domain (see \ref{tab:acs_features}). Because most demographic and socioeconomic measures are not reported at the precinct level, we first assemble them at the census block group level and then aggregate them to precincts using population-weighted interpolation based on total population. These variables capture population composition, socioeconomic status, and housing conditions, including total population, sex, age structure, racial and ethnic composition, educational attainment, median household income, unemployment, limited English proficiency, housing tenure, housing cost burden, rent, home value, housing age, overcrowding, and long commuting time. Using block groups as the source geography improves spatial precision relative to direct allocation from coarser census units and provides a consistent basis for precinct-level estimation.

\begin{table}[htp]
\centering
\small
\begin{tabular}{p{3cm}l}
\toprule
\textbf{Variable} & \textbf{Description} \\

\midrule
\multicolumn{2}{l}{\textit{Demographics}} \\
\midrule
POP & Log total population \\
Female & Log percent female \\
MedAge & Median age \\
Black & Percent Black (log transformed) \\
Hisp & Percent Hispanic (log transformed) \\
U18 & Percent under age 18 \\
O65 & Percent age 65 and older \\

\midrule
\multicolumn{2}{l}{\textit{Socioeconomic}} \\
\midrule
LtHS & Percent without high school diploma \\
Income & Median household income \\
Unemp & Unemployment rate (log transformed) \\
Eng & Limited English proficiency (log transformed) \\

\midrule
\multicolumn{2}{l}{\textit{Housing}} \\
\midrule
Apt & Buildings with $\geq 10$ units \\
Mobile & Mobile homes \\
OwnCost30 & Home owner housing cost burden ($\geq 30\%$) \\
RentCost30 & Renter housing cost burden ($\geq 30\%$) \\
Renters & Renter-occupied units \\
MedRent & Median rent \\
HomeVal & Median home value (log transformed) \\
YearBuilt & Median year built (log transformed) \\
Pre80 & Built before 1980 \\

\midrule
\multicolumn{2}{l}{\textit{Environmental}} \\
\midrule
NDVI & Vegetation index \\
LST & Land surface temperature (log transformed) \\
GDP & GDP (log-transformed) \\
NTL & Nighttime light intensity (log transformed) \\
Imperv & Impervious surface (log transformed) \\
Solo & Overcrowding (>1 person per room, log transformed) \\
Commute60 & Commute time $\geq 60$ minutes \\

\midrule
\multicolumn{2}{l}{\textit{Population / Political}} \\
\midrule
BVAP & Black voting-age population (log transformed) \\
G20Dem & Democratic vote share (2020) \\

\bottomrule
\end{tabular}
\caption{Precinct-level features grouped according to the modeling group feature set. 
All variables are derived from ACS or raster data sources as described in the text. 
Prior to modeling, selected variables are log transformed to reduce skewness 
(e.g., population, GDP, nighttime lights, impervious surface, and selected percentage variables), and all features are standardized to zero mean and unit variance.}
\label{tab:acs_features}
\end{table}

To complement these census-based measures, we incorporate remotely sensed indicators that capture landscape, environmental, and development characteristics. Annual mean land surface temperature (LST) for 2020 is derived from MODIS MOD11A1~\cite{wan2015mod11a1} after applying quality-control filters and converting the original product to degrees Celsius. Annual mean normalized difference vegetation index (NDVI) is obtained from MODIS MOD13Q1~\cite{didan2015mod13q1}, also distributed through the LP DAAC under the same open-use terms, using only good-quality pixels identified from the QA layer. Nighttime light intensity (NTL) is derived from the VIIRS nighttime lights V.2 annual composite~\cite{elvidge2021viirs}, produced by the Earth Observation Group and available on an open-access basis, as the annual mean for 2020, while retaining valid zero values to preserve meaningful low-light observations. We also include gridded GDP as an indicator of local economic intensity and development. GDP for 2020 is extracted from the Kummu et al.\ gridded GDP dataset~\cite{kummu2018gdp} and aggregated to precincts using area-weighted sums. These raster-based variables are extracted and processed in Google Earth Engine~\cite{gorelick2017gee} using the \texttt{reduceRegions()} function.

Together, these data provide a harmonized precinct-level dataset that captures demographic structure, socioeconomic conditions, environmental exposure, vegetation, nighttime activity, and development intensity, enabling a multidimensional assessment of redistricting patterns.

\paragraph{Data licenses and terms of use.}
U.S.\ Census ACS data are in the public domain. MODIS MOD11A1 and MOD13Q1 products are distributed through the LP DAAC with no restrictions on subsequent use, sale, or redistribution. VIIRS nighttime lights V.2 are available on an open-access basis from the Earth Observation Group, Payne Institute for Public Policy. The Kummu et al.\ gridded GDP dataset is released under CC0 (public domain). All data are used solely for non-commercial academic research.

North Carolina redistricting shapefiles, demographic attributes (including BVAP), and election data were obtained from the Quantifying Gerrymandering group at Duke University and are publicly available through the project repository (\url{https://quantifyinggerrymandering.pages.oit.duke.edu/codedoc/}). These datasets were constructed by merging U.S.\ Census data with North Carolina State Board of Elections returns and are publicly available through the project repository. The congressional and legislative districting plans analyzed in this study were obtained from the North Carolina General Assembly redistricting portal (\url{https://www.ncleg.gov/redistricting}).

\subsection{Additional Results}

\begin{figure}
    \centering
    \includegraphics[width=\linewidth]{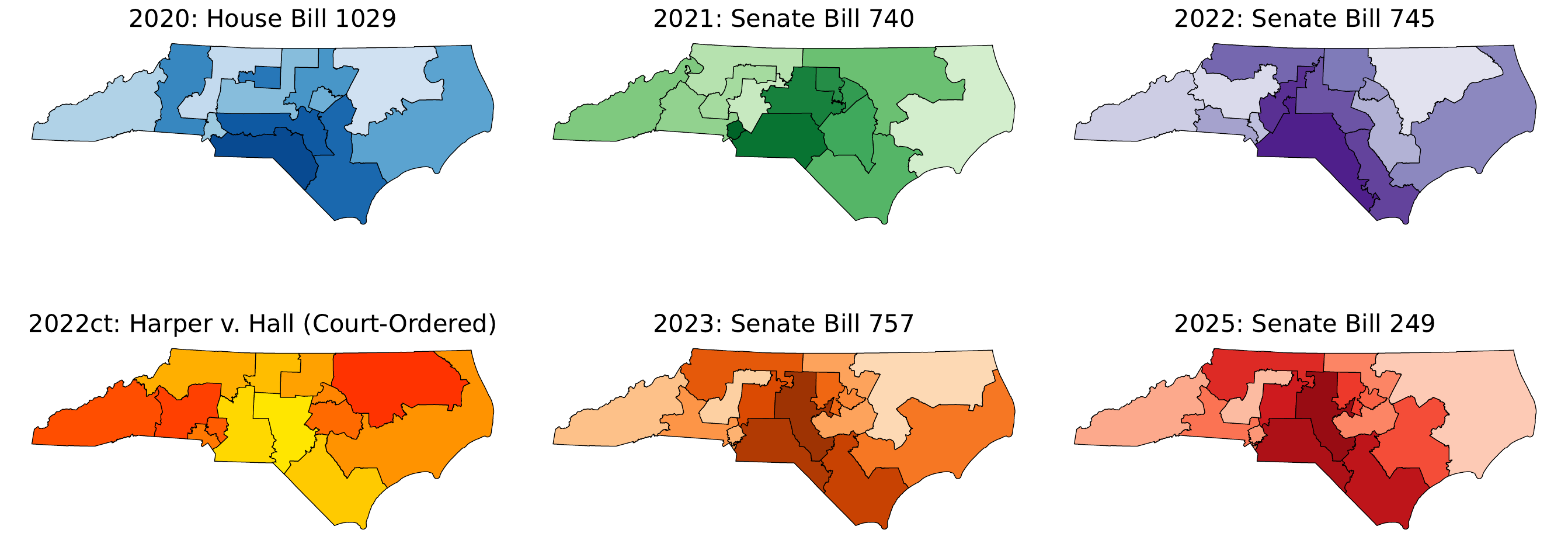}
    \caption{North Carolina Congressional Redistricting Plans (2020-2025)}
    \label{fig:ncmaps}
\end{figure}

Figure~\ref{fig:dendrograms} shows hierarchical clustering of the six NC redistricting plans under all six distance variants: GW, FGW, fsFGW (Lasso), fsFGW (Ridge), fsFGW (Simplex), and fsFGW (Group Simplex). The two-cluster structure separating {Plan21, Plan23, Plan25} from {Plan20, Plan22, Plan22ct} is consistent across all fsFGW modes, with Plans 23 and 25 merging at notably low distance in every case. Plan 21 is consistently isolated as the last to merge within its group. GW is the only variant that places Plan 20 with Plan 21, 23, and 25, confirming that incorporating electoral features systematically shifts Plan 20 toward Plans 22 and 22ct across all regularization strategies.

\begin{figure}[htbp]
    \centering

    % Row 1
    \begin{subfigure}{0.45\textwidth}
        \centering
        \includegraphics[width=\linewidth]{Images/GW_dendrogram.pdf}
    \end{subfigure}\hfill
    \begin{subfigure}{0.45\textwidth}
        \centering
        \includegraphics[width=\linewidth]{Images/FGW_dendrogram.pdf}
    \end{subfigure}

    \begin{subfigure}{0.45\textwidth}
        \centering
        \includegraphics[width=\linewidth]{Images/fsFGW_Lasso_dendrogram.pdf}
    \end{subfigure}\hfill
    \begin{subfigure}{0.45\textwidth}
        \centering
        \includegraphics[width=\linewidth]{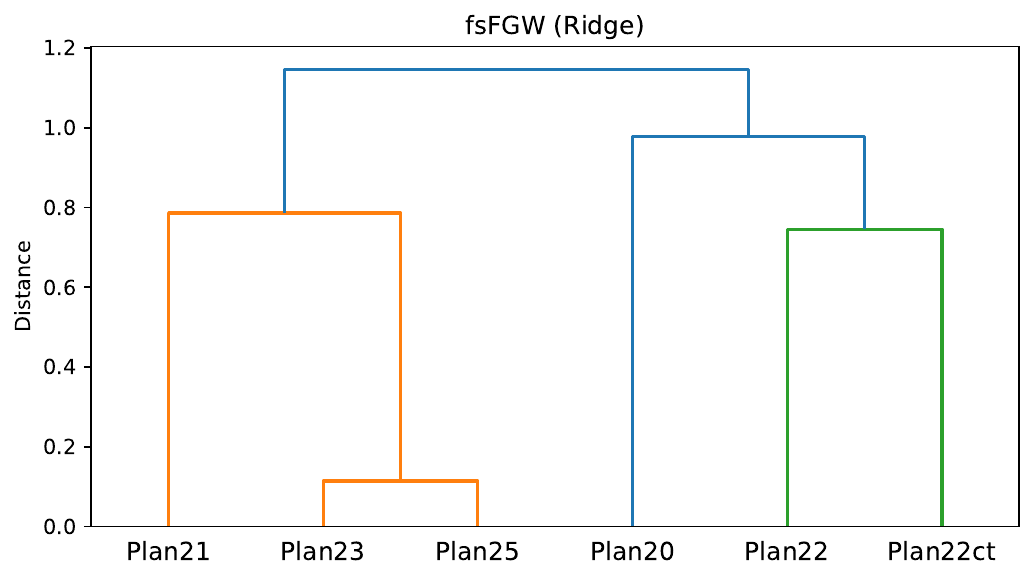}
    \end{subfigure}
    
    \begin{subfigure}{0.45\textwidth}
        \centering
        \includegraphics[width=\linewidth]{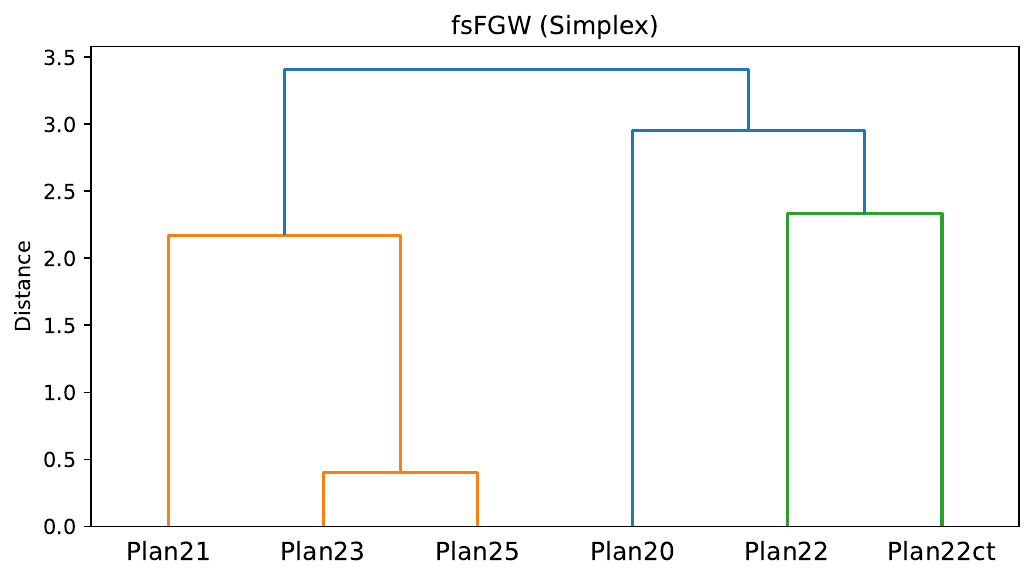}
    \end{subfigure}\hfill
    \begin{subfigure}{0.45\textwidth}
        \centering
        \includegraphics[width=\linewidth]{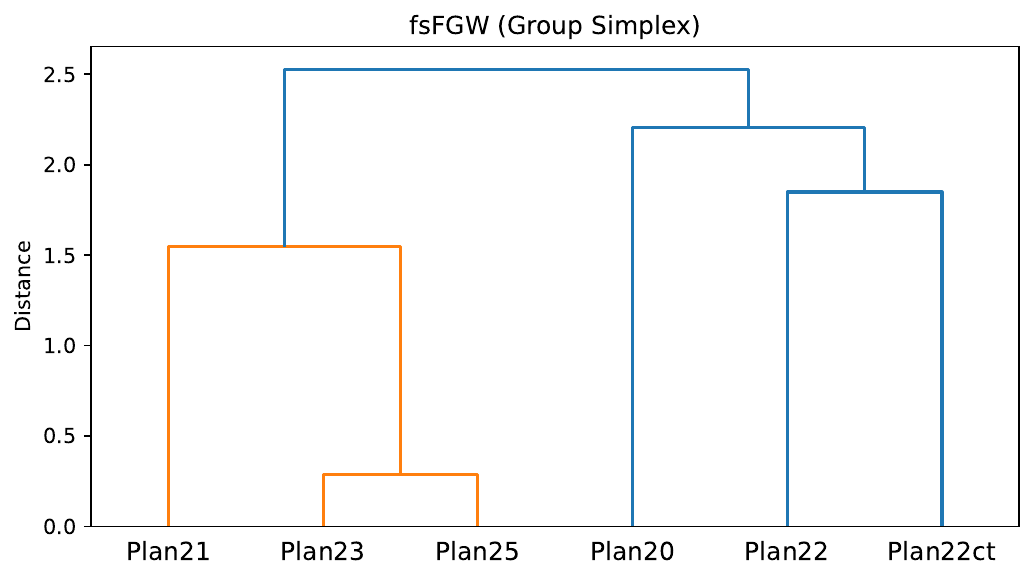}
    \end{subfigure}

    \caption{Hierarchical Clustering of NC Congressional Maps. }
    \label{fig:dendrograms}
\end{figure}

Figure~\ref{fig:ncweights22v22ct} reports the top mean suppression weights across all four fsFGW modes for the Plan 22 vs. Plan 22ct comparison. The simplex mode concentrates suppression on Mobile, Black, BVAP, and Eng, consistent with the main text. Ridge produces broadly elevated weights across many features with less sparsity, with Apt, Mobile, Solo, and NTL leading. Lasso yields intermediate sparsity, with Mobile, Apt, and Eng most suppressed. The groupwise simplex spreads weight more evenly across socioeconomic and demographic groups, with LtHS, Eng, Unemp, and Income leading. Demographic minority and housing characteristics consistently appear among the most suppressed features in the penalty-based and simplex modes, while the groupwise simplex additionally emphasizes socioeconomic indicators.

\begin{figure}[htbp]
    \centering

    \begin{subfigure}{0.8\textwidth}
        \centering
        \includegraphics[width=\linewidth]{Images/simplex_Plan22vsPlan22ct_mean_suppression_weights.pdf}
    \end{subfigure}
    
    \begin{subfigure}{0.8\textwidth}
        \centering
        \includegraphics[width=\linewidth]{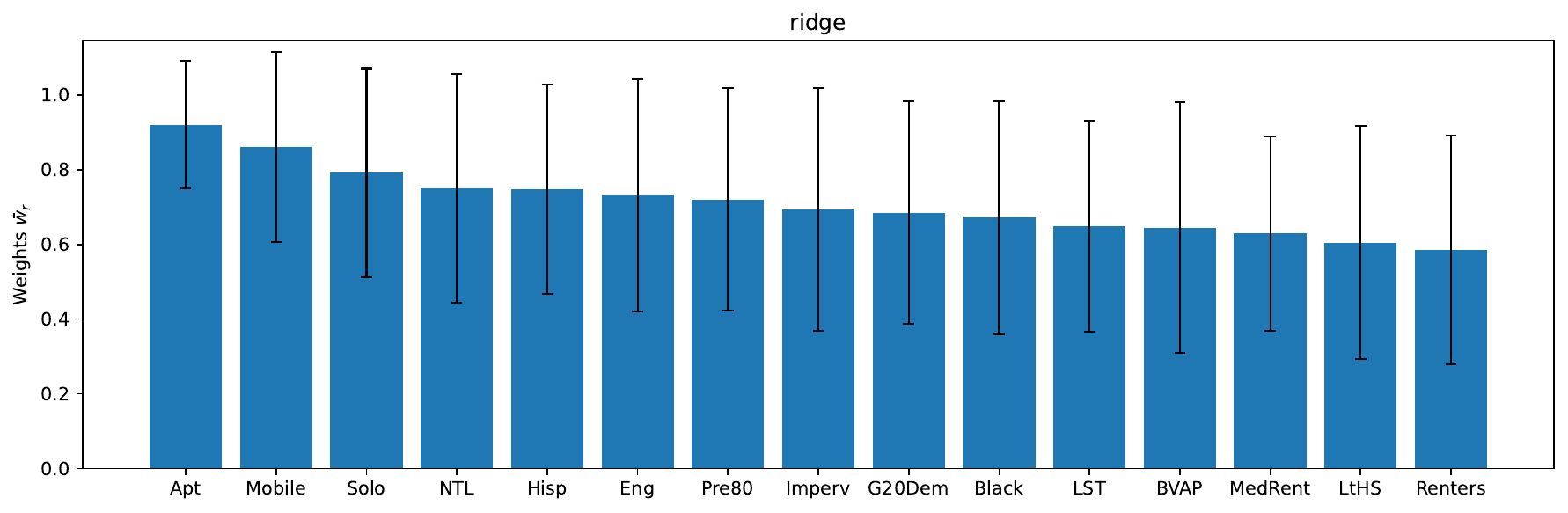}
    \end{subfigure}
    
    \begin{subfigure}{0.8\textwidth}
        \centering
        \includegraphics[width=\linewidth]{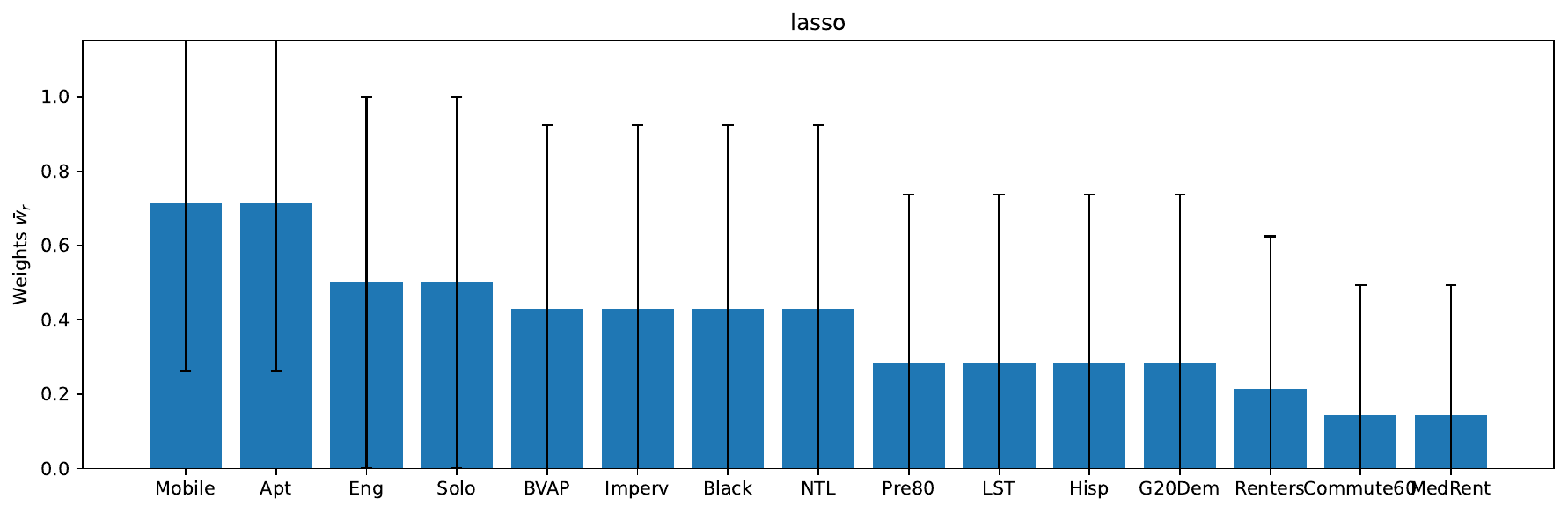}
    \end{subfigure}

    \begin{subfigure}{0.8\textwidth}
        \centering
        \includegraphics[width=\linewidth]{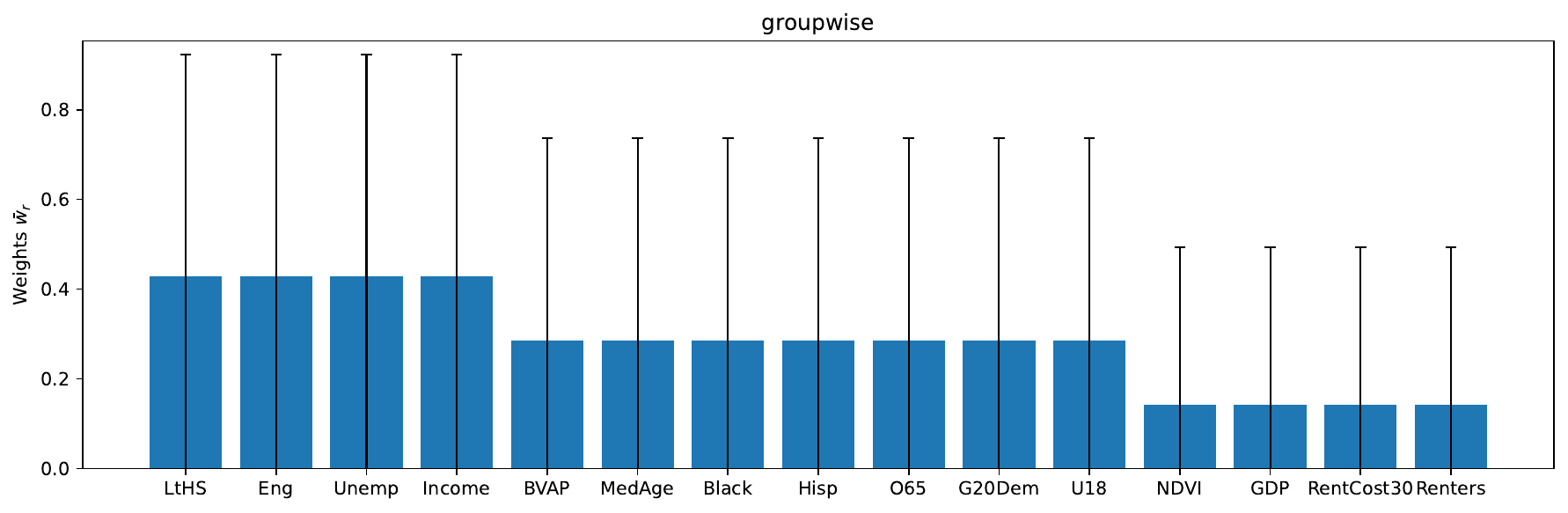}
    \end{subfigure}

    \caption{Top 10 mean suppression weights ($\pm$ 1 standard deviation) for Plan 22 vs. Plan22ct.}
    \label{fig:ncweights22v22ct}
\end{figure}

Figure~\ref{fig:ncweights23v25} shows suppression weights for each district pair under all four fsFGW modes when comparing Plan 23 and Plan 25. Lasso confirms the sparsity reported in the main text: only d0 and d7 receive nonzero weight. Ridge produces a denser pattern, with d0 and d7 again showing the strongest suppression but with nonzero weights spread across additional district pairs and features, reflecting Ridge's continuous rather than binary suppression. Simplex concentrates suppression entirely on BVAP for most active district pairs, as expected from its winner-take-all weight update. Groupwise Simplex suppresses BVAP and the demographic group broadly across nearly all district pairs, reflecting the coarser group-level granularity. Together, the four modes converge on BVAP and racial composition as the primary axis of difference, with Lasso providing the most interpretable and localized signal.

\begin{figure}[htbp]
    \centering

    % Row 1
    \begin{subfigure}{0.8\textwidth}
        \centering
        \includegraphics[width=\linewidth]{Images/Lasso_weights_23v25.pdf}
    \end{subfigure}
    
    \begin{subfigure}{0.8\textwidth}
        \centering
        \includegraphics[width=\linewidth]{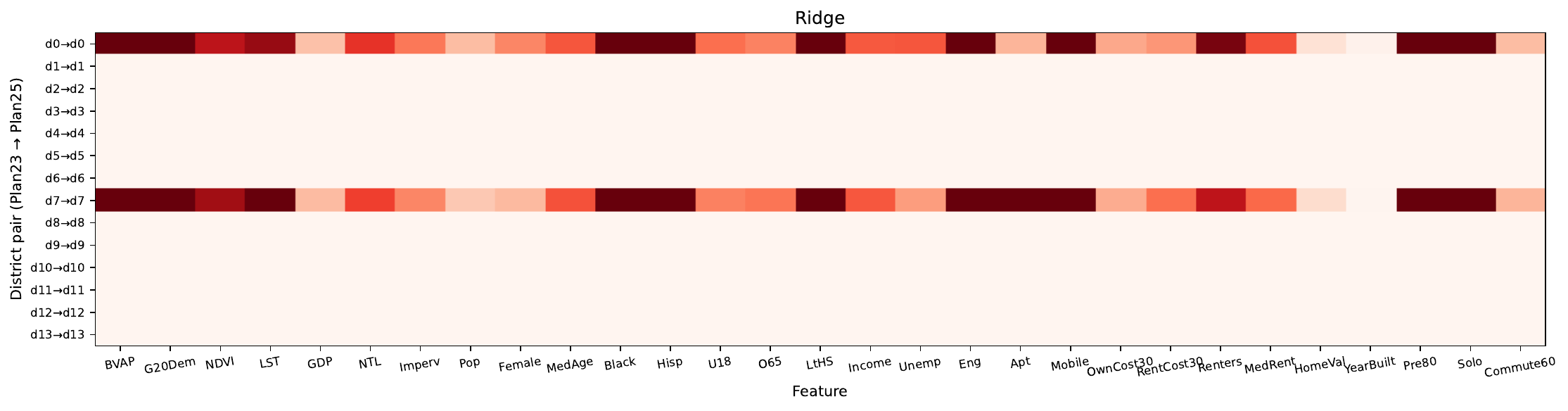}
    \end{subfigure}
    
    \begin{subfigure}{0.8\textwidth}
        \centering
        \includegraphics[width=\linewidth]{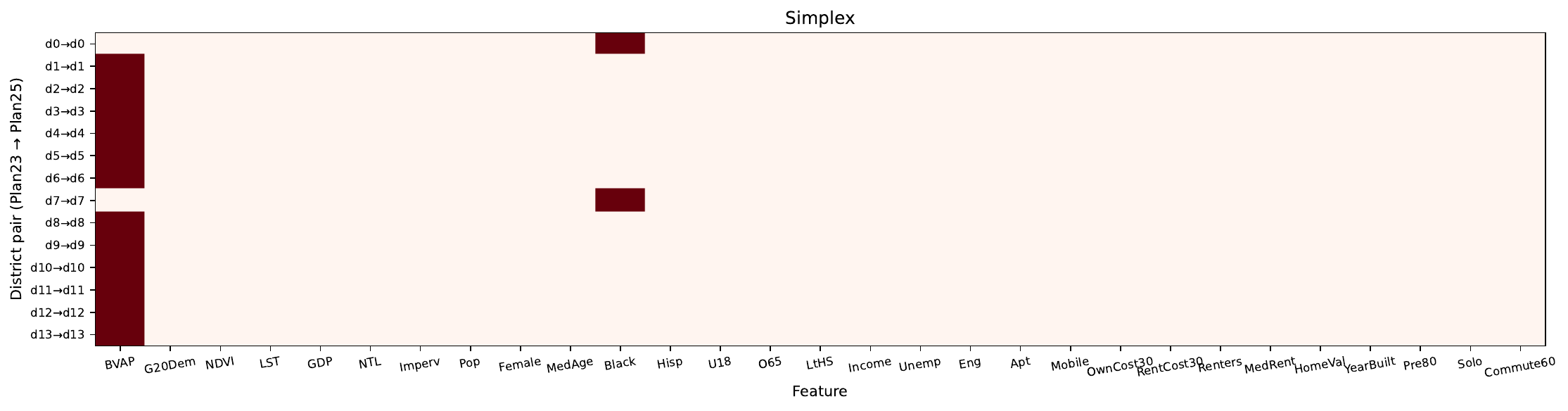}
    \end{subfigure}
    
    \begin{subfigure}{0.8\textwidth}
        \centering
        \includegraphics[width=\linewidth]{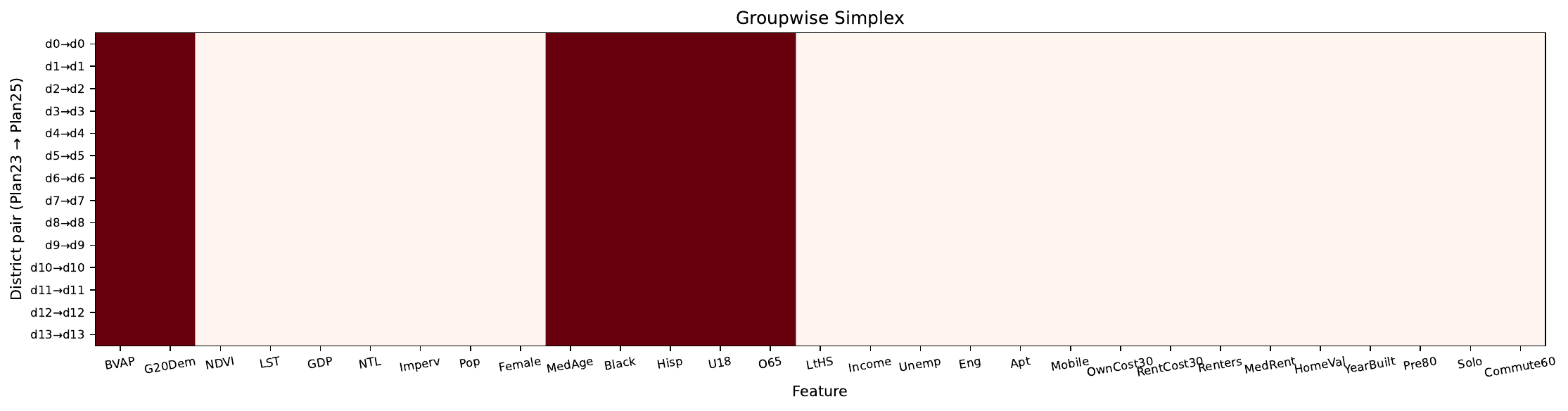}
    \end{subfigure}

    \caption{Weights for each district pair compared in Plan 23 vs. Plan25.}
    \label{fig:ncweights23v25}
\end{figure}

Figure~\ref{fig:featurecluster} shows the result of clustering features by their suppression patterns from comparing Plan22 and Plan22ct with fsFGW (Lasso). The dendrogram reveals three broad feature clusters. The first groups racial, ethnic, and linguistic minority characteristics (Black, BVAP, Hisp, Eng, Solo, Pre80). The second cluster captures urbanization and political indicators (Apt, Mobile, NTL, Imperv, G20Dem, LST), which share similar suppression patterns likely due to their common correlation with urban density. The third cluster comprises socioeconomic and housing cost variables (OwnCost30, YearBuilt, LtHS, RentCost30, Unemp, Commute60, GDP, HomeVal, MedRent, U18, Income, Renters).

\begin{figure}
    \centering
    \includegraphics[width=\linewidth]{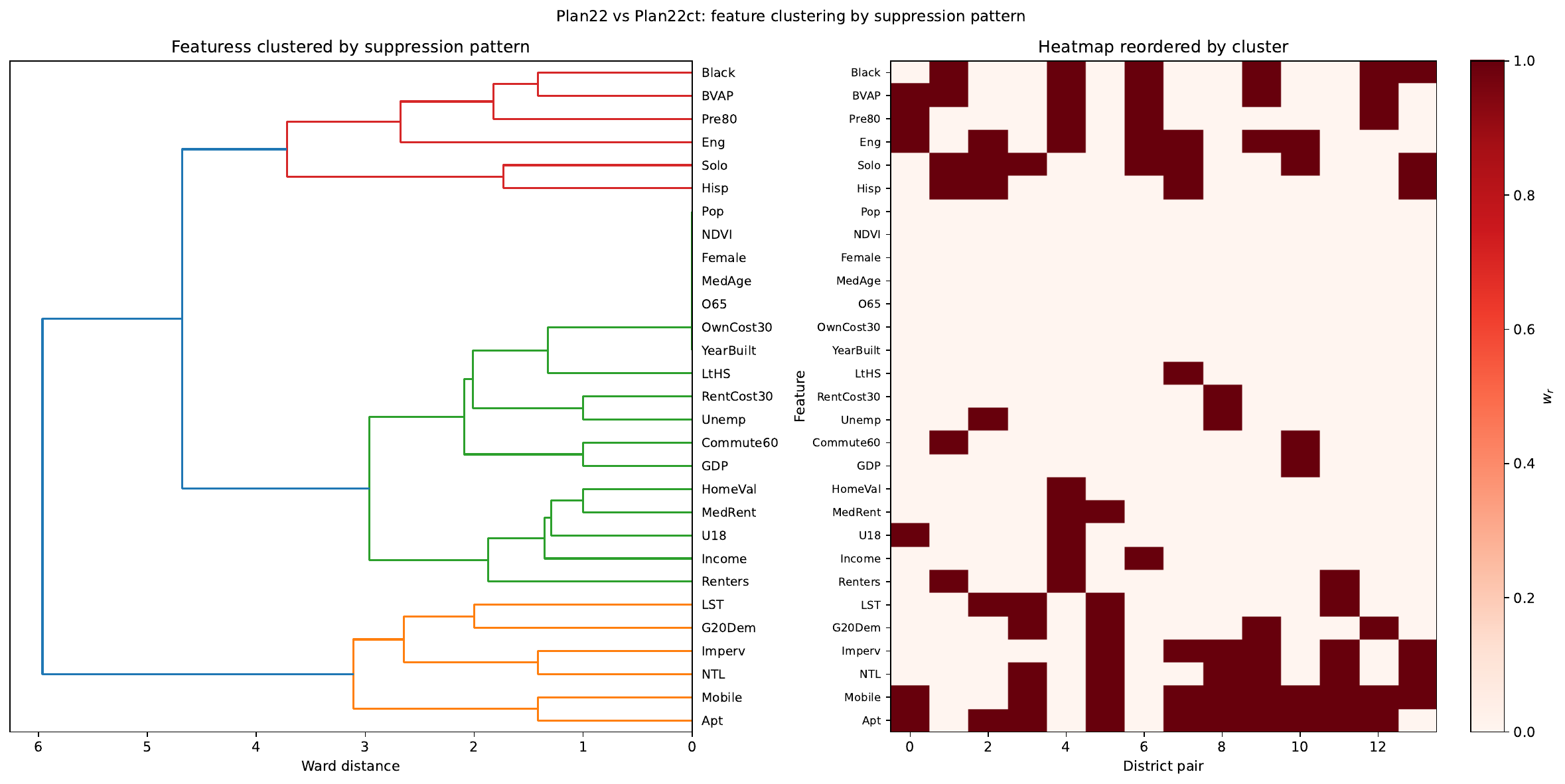}
    \caption{fsFGW (Lasso) suppression weight heatmap for Plan 22 vs. Plan 22ct, with features clustered by Ward linkage on their suppression patterns (left) and the corresponding heatmap reordered by cluster (right). Three feature clusters emerge: racial, ethnic, and linguistic minority characteristics (red); urbanization and political indicators (orange); and socioeconomic and housing cost variables (green).}
    \label{fig:featurecluster}
\end{figure}

\end{document}